\documentclass{article} \usepackage{iclr2027_conference,times}
\usepackage{amsmath,amssymb,amsthm,mathtools}
\newtheorem{proposition}{Proposition}

\newtheorem{remark}{Remark}
\newtheorem{assumption}{Assumption}

\usepackage{amsmath,amsfonts,bm}

\def\eqref#1{equation~\ref{#1}}

\def\1{\bm{1}}

\DeclareMathAlphabet{\mathsfit}{\encodingdefault}{\sfdefault}{m}{sl}
\SetMathAlphabet{\mathsfit}{bold}{\encodingdefault}{\sfdefault}{bx}{n}

\usepackage{booktabs}
\usepackage{url}
\usepackage{graphicx}
\usepackage{multirow}
\usepackage{algorithm}
\usepackage{algpseudocode}
\usepackage{enumitem}
\usepackage[table]{xcolor}
\usepackage{hyperref}
\hypersetup{
  hidelinks,
  pdftitle={Learn Feasibility Once, Optimize All Objectives: Derivative-Free Diffusion Models for Chance-Constrained Programming},
  pdfauthor={Ziwen Liu, Yan Liu, Congying Han, Tiande Guo, Yao Yan, Weichen Zhao}
}

\title{Learn Feasibility Once, Optimize All Objectives: Derivative-Free Diffusion Models for Chance-Constrained Programming}

\author{
Ziwen Liu$^{1}$, Yan Liu$^{2}$, Congying Han$^{3}$, Tiande Guo$^{3}$, Yao Yan$^{4}$\thanks{Co-corresponding authors: Yao Yan and Weichen Zhao.}, Weichen Zhao$^{2}$\footnotemark[1] \\
$^{1}$School of Mathematical Sciences, Beijing University of Posts and Telecommunications \\
$^{2}$School of Statistics and Data Science, Nankai University \\
$^{3}$School of Mathematical Sciences, University of the Chinese Academy of Sciences \\
$^{4}$School of Statistics and Data Science, Southwest University of Finance and Economics
}

\iclrfinalcopy 
\begin{document}
\maketitle
\fancyhead{}
\renewcommand{\headrulewidth}{0pt}
\begin{abstract}
Chance-constrained programs (CCPs) optimize decisions under uncertainty by limiting the probability of constraint violation. Despite advances in traditional and learning-based approaches, optimizing non-convex or non-smooth objectives and adapting to different objectives under fixed chance constraints remain challenging. In this paper, we propose a \textbf{D}erivative-free \textbf{D}iffusion-based framework that \textbf{D}isentangles constraint modeling from objective optimization, termed \textbf{D$^3$Opt}. We learn the chance-feasible structure once, independently of any particular objective, by training a risk-conditioned diffusion model solely on constraint-filtered decisions and freezing it as a reusable prior for post-specified objectives. At inference time, we propose an annealed, particle-based Feynman--Kac correction along the frozen reverse diffusion process to optimize post-specified objectives using only function evaluations. This enables derivative-free optimization of non-convex and non-smooth objectives without objective-specific retraining. We prove that the correction preserves feasibility when this property holds for the frozen prior, and derive an optimization-error bound separating learned-prior coverage, finite-particle approximation, and finite-temperature effects. Experiments on linear Gaussian CCPs, objective-transfer tasks, and chance-constrained economic dispatch demonstrate effective optimization across smooth and non-smooth objectives, including non-convex cases, and objective generalization under fixed chance constraints without retraining.
\end{abstract}
\section{Introduction}

Chance-constrained programming provides a principled framework for decision making under uncertainty by requiring constraints to hold with a prescribed probability. It has applications in economic dispatch~\citep{wang2017cced,yang2021cced}, optimal power flow~\citep{bienstock2014chance}, vehicle routing~\citep{dinh2018exact}, and production planning~\citep{zhang2023new}. Despite extensive research, evaluating probabilistic constraints and optimizing over the resulting non-convex feasible regions can remain computationally challenging~\citep{nemirovski2007convex,pagnoncelli2009sample}. In many applications, the uncertainty model and reliability requirements remain fixed while operational objectives or preferences change~\citep{cao2013chance,ren2026stochastic}. These tasks share the same chance-feasible structure, even though they favor different decisions. This observation raises the first question:

\textit{Q1: Can feasibility be learned once and reused across objectives?}

Learning-based approaches use data to guide optimization, and diffusion-based methods learn decision distributions that support constrained solution generation~\citep{kong2025constrained,zhang2025ggdopt,ding2025exploring}. However, the extent to which these distributions can be reused across objectives depends on their construction. GGDOpt~\citep{zhang2025ggdopt}, for example, trains on solutions of objective-dependent restricted problems, so its training data are tailored to a particular objective rather than constructed solely from constraint information. To remove this dependence, we separate learning \emph{where feasible decisions lie} from deciding \emph{which feasible decisions are preferred}. We generate training decisions using only constraint information and train a risk-conditioned diffusion model without evaluating the objective. Once trained, the model is frozen and reused across post-specified objectives under fixed chance constraints, without objective-specific data generation.

Reusing the learned constraint prior across objectives also requires handling the different optimization challenges these objectives introduce. In practical applications, objectives may be non-convex, non-smooth, or accessible only through function evaluations~\citep{masutomi2013evolutionary,guesmi2020chaotic}. Since guidance based on objective gradients or Hessians cannot cover all such cases, a second question follows:

\textit{Q2: How can the frozen constraint prior support optimization of non-convex and non-smooth objectives using only function evaluations?} 

We introduce an annealing path that progressively strengthens the objective preference along the frozen reverse diffusion process and implement it through particle-based Feynman--Kac correction (FKC)~\citep{delmoral2004feynman,skreta2025fkc,singhal2025fksteering}. The correction reweights and resamples particles using objective values evaluated at denoised decision predictions, concentrating the search toward low-objective regions without requiring objective gradients or Hessians. The same frozen model can therefore support derivative-free optimization across different objectives.

By combining the reusable constraint prior with the inference-time Feynman--Kac correction, we obtain \textbf{D$^3$Opt}, a \textbf{D}erivative-free \textbf{D}iffusion-based framework that \textbf{D}isentangles constraint modeling from objective optimization. Our main contributions are:
\begin{itemize}[leftmargin=1.5em]
\item We propose D$^3$Opt, which learns a reusable risk-conditioned diffusion prior solely from constraint information and supports objective generalization without retraining.
\item We develop an annealed, particle-based Feynman--Kac correction that optimizes post-specified non-convex and non-smooth objectives using only function evaluations.
\item We establish feasibility preservation under an almost-surely feasible prior and derive an error bound separating prior coverage, finite-particle, and finite-temperature effects.
\item We demonstrate solution quality and objective generalization on linear Gaussian CCPs, objective-transfer tasks, and chance-constrained economic dispatch.
\end{itemize}

\section{Related Work}
\textbf{Chance-constrained optimization.} Classical methods for CCPs include deterministic reformulations, conservative convex approximations, and sample average approximation (SAA)~\citep{nemirovski2007convex,pagnoncelli2009sample}. The non-convexity and discontinuity of empirical chance constraints motivate mixed-integer, difference-of-convex, and zero--one optimization methods~\citep{wang2023proximal,zhou2024zeroone}. Derivative-free solvers such as differential evolution and mesh adaptive direct search also accommodate non-smooth objectives~\citep{storn1997differential,audet2006mads}. These methods optimize a specified objective for each problem instance rather than learning a prior that can be reused across objectives.

\textbf{Diffusion models for constrained optimization.} Diffusion-based methods learn decision distributions and use generated samples to search under constraints~\citep{ding2025exploring,chen2024opportunities}. DiffOPT~\citep{kong2025constrained} represents unknown constraints through a learned data distribution and combines it with an objective-induced Boltzmann factor. It uses gradient-guided sampling for differentiable objectives and iterative importance sampling for non-differentiable ones. GGDOpt~\citep{zhang2025ggdopt} addresses CCPs by training on solutions of objective-dependent restricted problems and applying first- or second-order guidance at inference. Our framework learns a risk-conditioned prior from constraint-only data and applies annealed Feynman--Kac particle correction to the frozen sampler, adapting it to new objectives using objective values without retraining.

\section{Preliminaries}

\subsection{Problem Definition}
\label{sec:problem_definition}
We consider the chance-constrained programming:
\begin{equation} \min_{x\in\mathcal X} f(x)\qquad\mathrm{s.t.}\qquad \mathbb P_{\xi\sim\mathbb P}\bigl(g(x,\xi)\leq0\bigr)\geq1-\rho, \label{eq:ccp} \end{equation}
where $\mathcal X\subseteq\mathbb R^d$ is the decision domain, $f:\mathcal X\to\mathbb R$ is the objective, $\xi$ represents uncertainty, and $\rho\in(0,1)$ is the admissible violation probability. The chance-feasible set is
\begin{equation} 
\mathcal X_\rho=\left\{x\in\mathcal X:\mathbb P_{\xi\sim\mathbb P}\bigl(g(x,\xi)\leq0\bigr)\geq1-\rho\right\}. \label{eq:feasible_set} 
\end{equation}
Following the optimization-as-sampling perspective~\citep{kong2025constrained,zhang2025ggdopt}, let $p_{\mathrm F}(\cdot\mid\rho)$ be a reference distribution supported on $\mathcal X_\rho$. For inverse temperature $\beta\geq0$, we define
\begin{equation} \pi_\beta(x\mid\rho)=\frac{p_{\mathrm F}(x\mid\rho)e^{-\beta f(x)}}{Z_\beta(\rho)},\qquad Z_\beta(\rho)=\int_{\mathcal X}p_{\mathrm F}(x\mid\rho)e^{-\beta f(x)}\,dx, \label{eq:tilted_target} \end{equation}
assuming $0<Z_\beta(\rho)<\infty$. Under suitable regularity conditions and positive reference mass in neighborhoods of feasible global minimizers, $\pi_\beta$ concentrates on these minimizers as $\beta\to\infty$.
\subsection{Diffusion Model}
\label{sec:prelim_diffusion}
Score-based diffusion models learn to reverse a stochastic noising process using the score of its marginals~\citep{song2021scorebased}. We write the learned reverse process in generative time as:
\begin{equation}
    dX_t=b_{\theta,t}(X_t,\rho)\,dt+\sigma_t\,dB_t,
    \qquad
    X_0\sim p_{\mathrm{ref}},
\label{eq:reverse_diffusion}
\end{equation}
where $b_{\theta,t}$ contains the learned score contribution. We denote its path law by $\mathbb Q_\theta(\cdot\mid\rho)$ and its terminal distribution by $p_\theta(\cdot\mid\rho)$. Beyond sampling from the pretrained distribution itself, a growing line of work studies inference-time tilting toward an unnormalized target
\begin{equation*}
\label{eq:diffusion_tilt}
    \pi_{\theta,\beta}(x\mid\rho)
    \propto
    p_\theta(x\mid\rho)\exp[\beta r(x)].
\end{equation*}
Recent methods construct such information through value functions, denoised predictions, or lookahead rewards~\citep{li2025svdd,kim2026lookahead}. Other approaches instead control the proposal dynamics or combine transport with particle reweighting \citep{ren2026driftlite,skreta2025fkc}.

For optimization, however, a fixed tilt is generally not the most natural viewpoint. As discussed in Section~\ref{sec:problem_definition}, approaching the global optimum corresponds to the low-temperature limit $\beta\to\infty$, which motivates an annealed sequence $0=\beta_0<\beta_1<\cdots<\beta_K$ of progressively sharper target distributions. This creates a sequential change-of-measure problem: the sampler must gradually redistribute probability mass from the pretrained diffusion prior toward increasingly objective-favored regions. FKC provides a method for this purpose by representing the sequence of target changes.
\subsection{Feynman--Kac Correction}
\label{sec:prelim_fkc}
{\paragraph{Feynman--Kac change of measure.} Let $\mathbb Q$ denote the path law of a baseline diffusion. Feynman--Kac (FK) transforms modify this law by assigning nonnegative multiplicative weights to trajectories~\citep{delmoral2004feynman,lelievre2010free}. Given a terminal potential $G$ and a running potential $V_t$, the transformed path law takes the form
\[
\frac{d \mathbb Q^{G,V}}{d\mathbb Q}(X_{0:T})\propto \exp\!\left(-\int_0^T V_t(X_t)\,dt\right)G(X_T).
\]
Thus, Feynman--Kac correction is fundamentally a change of measure on the baseline dynamics. It may be realized through a drift modification when the associated backward information function is available, or through particle reweighting without modifying the baseline dynamics. Further continuous-time details are provided in Appendix~\ref{app:fkc}.

\paragraph{Feynman--Kac correction with particles.}
Instead of modifying the diffusion drift, FKC can keep the baseline dynamics unchanged and adjust particle probabilities through multiplicative weights. This yields a transport--diffusion--reweighting view of the FK evolution. Sequential Monte Carlo (SMC) approximates this process by alternating particle propagation, reweighting, and resampling \citep{delmoral2004feynman}. This particle formulation underlies recent FK-based inference-time steering methods for diffusion models~\citep{skreta2025fkc,ren2026driftlite,ou2026smc}.

For our method, consider a discretization $0=t_0<\cdots<t_K=T$ and denote the baseline Markov kernels by $M_k:=P_{t_{k-1},t_k}$. Given nonnegative incremental FK potentials $G_k:\mathcal X\times\mathcal X\to\mathbb R_+$, the FK correction reweights each baseline trajectory $X_{0:K}$ by the multiplicative factor $\prod_{k=1}^K G_k(X_{k-1},X_k)$. Thus, the pretrained diffusion kernels retain their role as the transport mechanism, whereas the FK potentials redistribute probability mass toward a prescribed sequence of tilted distributions. } The formal discrete-time FK measures and their connection to the continuous formulation are given in Appendix~\ref{app:fkc}.

\section{\texorpdfstring{D$^3$Opt}{D3Opt}}
\label{sec:method}
We introduce D$^3$Opt, a derivative-free, diffusion-based framework that disentangles constraint modeling from objective optimization. Section~\ref{sec:objective_free_diffusion} describes how D$^3$Opt learns a reusable risk-conditioned diffusion prior from constraint-only data. Section~\ref{sec:derivative_free_fkc} presents the annealed, particle-based Feynman--Kac correction for optimizing post-specified objectives using only function evaluations. Section~\ref{sec:theoretical_analysis} establishes feasibility preservation under the frozen prior and analyzes the resulting optimization error.

\subsection{Objective-Free Feasible Diffusion Model}
\label{sec:objective_free_diffusion}
\paragraph{Constraint-only data generation.}
To decouple constraint learning from objective optimization, we construct the training dataset without evaluating any objective function or solving objective-dependent optimization problems. Candidate decisions are generated and retained using only the decision domain and chance-constraint information. The resulting dataset can therefore be reused when the objective changes while the constraints remain fixed.

For each problem, we choose an objective-independent proposal distribution $\nu$ on the deterministic decision domain $\mathcal X$. We independently draw a candidate $X\sim\nu$ and a target risk level $\rho\sim\omega$, where $\omega$ is a density on $\mathcal R=[\rho_{\min},\rho_{\max}]$. The proposal distribution may differ across problems to reflect their decision domains, but remains fixed across objectives for the same problem.

We assess each candidate using two independent scenario sets, $\Xi_{\mathrm{lab}}$ and $\Xi_{\mathrm{val}}$. These provide empirical estimates of the constraint-satisfaction probability $c(x)=\mathbb P_{\xi\sim\mathbb P}(g(x,\xi)\leq0)$:
\begin{equation*}
\widehat c_{\mathrm{lab}}(x)=\frac{1}{N_{\mathrm{lab}}}\sum_{s=1}^{N_{\mathrm{lab}}}\mathbb I\{g(x,\xi_{\mathrm{lab}}^{(s)})\leq0\},
\qquad \widehat c_{\mathrm{val}}(x)=\frac{1}{N_{\mathrm{val}}}\sum_{s=1}^{N_{\mathrm{val}}}\mathbb I\{g(x,\xi_{\mathrm{val}}^{(s)})\leq0\}.
\label{eq:empirical_condition}
\end{equation*}
For a chance constraint, the indicator requires all components of $g$ to be non-positive. We retain $(X,\rho)$ when
\begin{equation*} A_N(X,\rho)=\mathbb I\left\{\widehat c_{\mathrm{lab}}(X)\geq1-\rho,\ \widehat c_{\mathrm{val}}(X)\geq1-\rho\right\}=1.
\label{eq:acceptance_rule}  
\end{equation*}

Repeating this procedure produces $\mathcal D_F=\{(x^{(i)},\rho^{(i)})\}_{i=1}^{N_D}$, where each decision is paired with its sampled target risk. Conditional on the scenario sets, accepted decisions at a fixed risk level follow
\begin{equation*} p_{F,N}(x\mid\rho)\propto\nu(x)A_N(x,\rho). \label{eq:accepted_distribution} \end{equation*}
Both $\nu$ and $A_N$ are independent of the objective. The diffusion model thus learns from constraint-filtered decisions without a training preference for any particular objective.

\paragraph{Conditional diffusion training.} We instantiate the score-based diffusion framework introduced in Section~\ref{sec:prelim_diffusion} using a variance-preserving DDPM with the standard noise-prediction parameterization. For each objective-free feasible sample $(X_0,\rho)\in\mathcal D_F$, the forward process is
\begin{equation*}
    X_t
    =
    \sqrt{\bar\alpha_t}X_0
    +
    \sqrt{1-\bar\alpha_t}\,\epsilon,
    \qquad
    \epsilon\sim\mathcal N(0,I).
\end{equation*}
We train a risk-conditioned noise predictor $\epsilon_\theta(X_t,t,\rho)$ using the standard denoising objective
\begin{equation*}
\label{eq:diffusion_training}
    \mathcal L_{\mathrm{diff}}(\theta)
    =
    \mathbb E_{(X_0,\rho),t,\epsilon}
    \left[
        \left\|
            \epsilon-
            \epsilon_\theta(X_t,t,\rho)
        \right\|_2^2
    \right].
\end{equation*}
The resulting numerical reverse process defines conditional Markov kernels $M_k^\theta(x_{k-1},dx_k\mid\rho)$, $k=1,\ldots,K$, with intermediate marginals $q_k^\theta(\cdot\mid\rho)$ and terminal distribution $p_\theta(\cdot\mid\rho)$ approximating the objective-free feasible distribution learned from $\mathcal D_F$.

The objective function $f$ is absent from data generation, diffusion training, checkpoint selection, and model parameters. After training, the conditional reverse process is frozen and used as the baseline transport for the inference-time Feynman--Kac correction introduced in the next subsection.
\subsection{Derivative-Free Feynman--Kac Correction}
\label{sec:derivative_free_fkc}
After training, the diffusion model is frozen. We then introduce a post-specified objective $f:\mathcal X\rightarrow\mathbb R$, which may be non-convex, non-smooth, or available only through function evaluations.

\paragraph{Annealed derivative-free FKC.}
Consider the generative-time discretizations $0=t_0<t_1<\cdots<t_K=T$. For an intermediate particle state $X_k$, let $\widehat X_k:=\mathsf D_{\theta,k}(X_k,\rho)$ denote the terminal decision predicted by the frozen conditional diffusion model, with $\widehat X_K=X_K$ at the final step. Since the objective is defined on the terminal decision, we evaluate the current particle through $f(\widehat X_k)$. To progressively concentrate the particle system toward low-objective regions, we introduce an inverse-temperature schedule $0=\beta_0\leq\beta_1\leq\cdots\leq\beta_K=\beta_{\max}$ and define the incremental FK potential
\begin{equation}
\label{eq:annealed_fk_potential}
    G_k(X_{k-1},X_k;\rho)
    \triangleq
    \exp\!\left[
        \beta_{k-1}f(\widehat X_{k-1})
        -
        \beta_k f(\widehat X_k)
    \right],
    \qquad k=1,\ldots,K.
\end{equation}
This potential requires only evaluations of the objective, in particular, no gradient or Hessian of $f$ is needed. Moreover, the incremental potentials telescope along the reverse diffusion trajectory
\begin{equation*}
\label{eq:terminal_tilt}
    \prod_{k=1}^{K}
    G_k(X_{k-1},X_k;\rho)=\exp\left[
        \beta_0 f(\widehat X_0)-
        \beta_K f(\widehat X_K)
    \right]=\exp\left[-\beta_K f(X_K)\right],
\end{equation*}
where $\beta_0=0$ and $\widehat X_K=X_K$. Hence the terminal FK marginal is $\pi_{\theta,\beta_K}(x\mid\rho) \propto p_\theta(x\mid\rho)\exp[-\beta_K f(x)]$. Thus, the intermediate denoised predictions determine how particle mass is redistributed during sampling, but do not alter the desired terminal change of measure. The chance constraint does not enter the FK potential through an explicit penalty, and the risk level $\rho$ instead enters through the frozen conditional diffusion model.

\paragraph{Interacting-particle approximation.}
We approximate the FK flow using $M$ interacting particles. Starting from $X_0^{(m)}\sim p_{\mathrm{ref}}$ and $w_0^{(m)}=1/M$, each particle is propagated through the frozen conditional reverse diffusion $X_k^{(m)} \sim M_k^\theta\!(X_{k-1}^{(m)},\cdot\mid\rho)$, $k=1,\ldots,K$. We use a prescribed inverse-temperature schedule to progressively increase the influence of the post-specified objective:
\begin{equation}
    \beta_k = \beta_{\max} \left(k/K\right)^\nu, \qquad k=0,\ldots,K,\quad \nu>0.
    \label{eq:prescribed_beta}
\end{equation}
Here, $\beta_{\max}$ controls the terminal objective tilt, while $\nu$ determines how rapidly the objective influence is introduced along the generative trajectory. The schedule is fixed before sampling and satisfies $\beta_0=0$ and $\beta_K=\beta_{\max}$. Using the incremental potential in \eqref{eq:annealed_fk_potential}, the particle weights are updated as
\begin{equation*}
    \widetilde w_k^{(m)} = w_{k-1}^{(m)} G_k( X_{k-1}^{(m)},X_k^{(m)};\rho ), \qquad w_k^{(m)} =\widetilde w_k^{(m)}/ \sum_{j=1}^{M}\widetilde w_k^{(j)}.
\end{equation*}
Although the temperature schedule is prescribed, resampling is performed adaptively according to the resulting weight degeneracy. Since the particle weights are normalized, their effective sample size is $\operatorname{ESS}_k =1/\sum_{m=1}^{M} (w_k^{(m)})^2$. Whenever $\operatorname{ESS}_k<\tau_{\mathrm{res}}M$, with $\tau_{\mathrm{res}}\in(0,1]$, we apply systematic resampling according to $\{w_k^{(m)}\}_{m=1}^{M}$ and reset all weights to $1/M$. Otherwise, the normalized weights are retained for the next propagation step. Thus, the prescribed temperature schedule controls the progressive objective tilt, while ESS is used only to determine when particle resampling is necessary.

\subsection{Feasibility Preservation and Error Analysis}
\label{sec:theoretical_analysis}
We next characterize the theoretical consequences of our method. We first show that, at the ideal measure level, FKC preserves any almost-sure feasibility already encoded by the frozen diffusion prior. We then analyze how finite-temperature concentration, approximation of the ideal feasible distribution, and finite-particle simulation affect the resulting optimizer.

\paragraph{Feasibility preservation under FKC.}
A key consequence of the proposed objective--constraint disentanglement is that the objective-only FKC cannot create probability mass outside the support of the frozen diffusion prior. Consequently, if the frozen diffusion is chance-feasible almost surely, the FK-tilted distribution remains chance-feasible almost surely, FKC itself does not introduce infeasibility. Let $\mathbb{Q}_\theta(\cdot\mid\rho)$ denote the path law of the frozen conditional reverse diffusion, and let $\mathbb{Q}_\theta^{\mathrm{FK}} (\cdot\mid\rho)$ denote the FK transform induced by the incremental potentials in \eqref{eq:annealed_fk_potential}.

\begin{proposition}[Feasibility preservation under FKC]
\label{prop:fk_feasibility}
Fix $\rho\in(0,1)$ and suppose that the Feynman--Kac path weight is nonnegative, integrable, and has a strictly positive normalizing constant, then $\mathbb{Q}_\theta^{\mathrm{FK}}$ is absolutely continuous with respect to $\mathbb{Q}_\theta(\cdot\mid\rho)$, i.e. $\mathbb{Q}_\theta^{\mathrm{FK}}(\cdot\mid\rho) \ll \mathbb{Q}_\theta(\cdot\mid\rho)$. Consequently, the corresponding terminal marginal satisfies
\begin{equation}
    \pi_{\theta,\beta_K}(\cdot\mid\rho)
    \ll
    p_\theta(\cdot\mid\rho).
    \label{eq:terminal_absolute_continuity}
\end{equation}
In particular, if the frozen diffusion is chance-feasible almost surely, i.e. $p_\theta(\mathcal X_\rho\mid\rho)=1$, then the FK-corrected terminal distribution also satisfies $\pi_{\theta,\beta_K}(\mathcal X_\rho\mid\rho)=1$.
\end{proposition}
Proof of Proposition~\ref{prop:fk_feasibility} is given in Appendix~\ref{app:fk_feasibility}. For the objective-only tilt,$ \frac{d\pi_{\theta,\beta_K}(\cdot\mid\rho)}{dp_\theta(\cdot\mid\rho)}(x)=\frac{\exp[-\beta_K f(x)]}{Z_{\theta,\beta_K}(\rho)}$, whenever $0<Z_{\theta,\beta_K}(\rho)<\infty$. Hence FKC only redistributes mass within the support of the frozen prior and cannot create mass on baseline null sets. The condition $p_\theta(\mathcal X_\rho\mid\rho)=1$ is an ideal support assumption and need not hold exactly under the relaxed finite-scenario construction and finite-data diffusion approximation in Section~\ref{sec:objective_free_diffusion}. Proposition~\ref{prop:fk_feasibility} isolates the effect of FK correction by showing that FKC does not introduce support outside that of the frozen prior.

\paragraph{Approximation and optimization error.}
Proposition~\ref{prop:fk_feasibility} shows that the objective-only Feynman--Kac correction preserves feasibility already encoded by the frozen diffusion prior. We next quantify the practical optimization error arising from incomplete coverage of the learned feasible distribution, finite inverse temperature, and finite-particle approximation.

For a fixed risk level $\rho$, define $f_\rho^\star \triangleq \inf_{x\in\mathcal X_\rho} f(x)$, and $f_{\theta,\rho}^\star \triangleq \operatorname*{ess\,inf}_{x\sim p_\theta(\cdot\mid\rho)} f(x)$ is the essential infimum of $f$ taken with respect to $p_\theta(\cdot\mid\rho)$. We further define the learned-prior coverage gap
\begin{equation*}
    \Delta_{\theta,\rho}
    \triangleq
    f_{\theta,\rho}^\star-f_\rho^\star.
    \label{eq:learned_coverage_gap}
\end{equation*}
Under almost-sure feasibility of the frozen diffusion, $\Delta_{\theta,\rho}\geq0$. In particular, $\Delta_{\theta,\rho}=0$ whenever the learned prior assigns positive probability to every arbitrarily near-optimal feasible objective sublevel. For any $\epsilon>0$, define the near-optimal mass of the learned prior by
\begin{equation}
    a_{\theta,\epsilon}
    \triangleq
    p_\theta\!\left(
        f(X)
        \leq
        f_{\theta,\rho}^\star+\epsilon
        \,\middle|\,\rho
    \right).
    \label{eq:near_optimal_mass_theta}
\end{equation}
By the definition of the essential infimum, $a_{\theta,\epsilon}>0$ for every $\epsilon>0$. Let $\eta_K$ denote the exact normalized terminal FK marginal and $\eta_K^M=\sum_{m=1}^M w_K^{(m)}\delta_{X_K^{(m)}}$ its normalized particle approximation.
\begin{assumption}[Regularity for practical error analysis]
\label{ass:practical_error}
Fix $\rho\in(0,1)$. Assume $p_\theta(\mathcal X_\rho\mid\rho)=1$, $f_{\theta,\rho}^\star>-\infty$. For every fixed finite inverse-temperature schedule $0=\beta_0\leq\cdots\leq\beta_K<\infty$, assume that the associated FK model is well defined and that its $M$-particle approximation satisfies $( \mathbb E[ | \eta_K^M(f)-\eta_K(f) |^2 ] )^{1/2} \leq C_{K,f}/\sqrt M$ for some finite $C_{K,f}$ independent of $M$.
\end{assumption}
The first condition is the feasibility premise of Proposition~\ref{prop:fk_feasibility}. The particle-error condition assumes the standard fixed-horizon $O(M^{-1/2})$ SMC rate, which holds for bounded observables and extends to unbounded test functions under suitable moment and stability  conditions~\citep{delmoral2004feynman,hu2011general,whiteley2012sequential}. Importantly, no differentiability, smoothness, or Hessian information of the post-specified objective $f$ is required.

\begin{proposition}[Practical optimization error]
\label{prop:optimization_error}
Under Assumption~\ref{ass:practical_error}, for every finite $\beta_K>0$, there exists a finite constant $C_{K,f}$, independent of $M$, such that
\begin{equation}
\label{eq:practical_optimization_error}
    \left|\mathbb E[\eta_K^M(f)]-f_\rho^\star\right|\leq\underbrace{\Delta_{\theta,\rho}}_{\text{coverage error}}+\underbrace{\frac{C_{K,f}}{\sqrt M}}_{\text{finite particles}}+\underbrace{\inf_{\epsilon>0}\left\{\epsilon+\frac{1}{\beta_K}\log\frac{1}{a_{\theta,\epsilon}}\right\}}_{\text{finite inverse temperature}}.
\end{equation}
\end{proposition}
Proof can be found in Appendix~\ref{app:optimization_error}. In practice, the SMC procedure returns a weighted terminal particle population, from which we output the one with the smallest objective value. Since $f(\widehat X_K^M):=\min_{1\le m\le M}f(X_K^{(m)})\leq \eta_K^M(f)$, the bound on $\mathbb E[\eta_K^M(f)]$ also provides an upper bound on the expected objective gap of the returned best particle. Moreover,
\begin{equation}
    \limsup_{\beta_K\rightarrow\infty}
    \limsup_{M\rightarrow\infty}
    \left|\mathbb E[\eta_K^M(f)]-f_\rho^\star
    \right|
    \leq
    \Delta_{\theta,\rho}.
    \label{eq:asymptotic_coverage_gap}
\end{equation}
In particular, when $\Delta_{\theta,\rho}=0$, the infinite-particle limit followed by the low-temperature limit recovers the global chance-constrained optimal value in expectation. The analysis experiments are presented in Appendix~\ref{app:ablation}.

\section{Experiments}
We evaluate D$^3$Opt on linear Gaussian optimization (Section~\ref{sec:linear_gaussian}), objective transfer to smooth and non-smooth objectives (Section~\ref{sec:linear_gaussian_generalization}), and economic dispatch with valve-point effects (Section~\ref{sec:cced_vpe}). Unless stated otherwise, stochastic methods are evaluated over $100$ runs. D$^3$Opt uses a U-Net score model~\citep{ronneberger2015u} and stochastic DDIM sampling~\citep{song2021denoising}, with the DDIM stochasticity parameter fixed at $\eta=1$ throughout the reported experiments. Implementation details and robust waveform experiments are provided in Appendices~\ref{app:implementation} and~\ref{app:robust_waveform}. All experiments use one NVIDIA RTX 5090 GPU.

\subsection{Linear Gaussian Chance-Constrained Optimization}
\label{sec:linear_gaussian}
We first consider the standard linear Gaussian chance-constrained optimization problem:
\begin{align}
\min_{\mathbf{x}\in\mathbb{R}^{n}}\quad
 \frac{1}{2}\mathbf{x}^{\top}\mathbf{x}
+\mathbf{b}^{\top}\mathbf{x},
\qquad
\text{s.t.}\quad
 \mathbb{P}_{\mathbf{c}}
\left\{
\mathbf{c}^{\top}\mathbf{x}+d\geq 0
\right\}
\geq 1-\rho, \label{eq:lg_constraint}
\end{align}
where $\mathbf{c}\sim\mathcal{N}(\bar{\mathbf{c}},\mathbf{I})$ and $\rho\in(0,1)$ denotes the prescribed violation-probability limit. Following~\cite{zhang2025ggdopt}, we set $\mathbf{b}=\bar{\mathbf{c}}=\mathbf{1}$, $d=1$, $n=8$, and $\rho=0.1$. The problem can be reformulated as second-order conic (SOC) problem, so we use the CVX~\citep{grant2008cvx} as the global reference. And we also compare our method with different types of SAA methods. We additionally evaluate GGDOpt First-order/Second-order~\citep{zhang2025ggdopt} and DiffOPT~\citep{kong2025constrained}. For DiffOPT, we consistently use its iterative importance-sampling variant throughout our experiments, including settings with differentiable objectives, so that objective access remains value-only.
\begin{table}[t!]
\centering
\caption{Comparison on the linear Gaussian chance-constrained problem~\eqref{eq:lg_constraint} with $n=8$ and $\rho=0.1$. Lower objective values are better.}
\label{tab1}
\resizebox{\textwidth}{!}{\begin{tabular}{lcccc}
\toprule
Method & Repeat & FvalMean $\downarrow$ & FvalStd $\downarrow$ & FvalMedian $\downarrow$ \\
\midrule
SOC-CVX~\citep{grant2008cvx}~\textbf{(Global Reference)}
& 1 & $-0.6586$ & $0$ & $-0.6586$ \\
\midrule
SAA-CVaR~\citep{nemirovski2007convex}
& 100 & $-0.5893$ & $0.0248$ & $-0.5869$ \\
SAA-MIP~\citep{pagnoncelli2009sample}
& 100 & $-0.6281$ & $0.0157$ & $-0.6318$ \\
SAA-PDCA~\citep{wang2023proximal}
& 100 & $-0.6389$ & $0.0314$ & $-0.6408$ \\
SAA-SNSCO~\citep{zhou2024zeroone}
& 100 & $0.8051$ & $3.4014$ & $-0.6371$ \\
\midrule
GGDOpt First-order~\citep{zhang2025ggdopt}
& 100 & $-0.6483$ & $0.0051$ & $-0.6488$ \\
GGDOpt Second-order~\citep{zhang2025ggdopt}
& 100 & $-0.6491$ & $0.0056$ & $-0.6503$ \\
DiffOPT ~\citep{kong2025constrained}
& 100 & $-0.5614$ & $0.0252$ & $-0.5608$ \\
\midrule
\textbf{D$^3$Opt}\ w/o FKC & 100 & $-0.1172$ & $0.0711$ & $-0.1247$  \\
\textbf{D$^3$Opt}
& 100 & $\mathbf{-0.6519}$ & $\mathbf{0.0015}$ & $\mathbf{-0.6517}$ \\
\bottomrule
\end{tabular}}
\end{table}

\textbf{Results.} On the linear Gaussian benchmark, Table~\ref{tab1} shows that the frozen constraint-only prior leaves substantial room for objective improvement. Without FKC, the sampler returns a mean objective value of $-0.1172$. Adding objective-value correction lowers it to $-0.6519$, only $0.0067$ above the exact SOC optimum $-0.6586$. D$^3$Opt also attains the lowest mean and median and the smallest standard deviation $0.0015$ among the evaluated diffusion methods. These gains are achieved at inference without objective gradients or objective-specific retraining.

\subsection{Objective Generalization under a Linear Gaussian Chance Constraint}
\label{sec:linear_gaussian_generalization}
We next evaluate objective generalization by reusing the same frozen constraint-only diffusion checkpoint from Section~\ref{sec:linear_gaussian} for objectives specified only after training. We consider a smooth anisotropic shifted quadratic (ASQ) objective and a non-convex, non-smooth four-well (MW4) objective, with respect to the same constraint as \eqref{eq:lg_constraint}. 

\textbf{Anisotropic shifted quadratic (ASQ) objective.} The ASQ objective can be described as:
\begin{equation} 
f_{\mathrm{ASQ}}(\mathbf{x})=\frac{1}{2}(\mathbf{x}-\boldsymbol{\mu})^{\top}\mathbf{Q}(\mathbf{x}-\boldsymbol{\mu}),\qquad \mathbf{Q}\succ 0  \label{eq:asq_objective} 
\end{equation}
The numerical parameters of \(Q\) and \(\mu\) are provided in Appendix~\ref{app:problem_setting}. Since $\mathbf{Q}\succ0$, the ASQ objective is smooth and strictly convex. Compared with the standard quadratic objective in Section~\ref{sec:linear_gaussian}, it changes both the unconstrained minimizer and the directional curvature while leaving the chance constraint unchanged. ASQ therefore provides a controlled test of whether the frozen constraint-only diffusion prior can adapt to a new smooth convex objective introduced only at inference time.
\begin{table}[t]
\centering
\caption{Objective-transfer results on ASQ and MW4 under the same linear Gaussian chance constraint. Each stochastic method is evaluated over 100 independent runs. Lower objective values and gaps are better.}
\label{tab:linear_transfer_results}
\resizebox{\textwidth}{!}{\begin{tabular}{llcccc}
\toprule
Task & Method & FvalMean $\downarrow$ & FvalStd $\downarrow$ & FvalBest $\downarrow$ & Optimality Gap (\%) $\downarrow$ \\
\midrule
\multirow{6}{*}{ASQ} 
& SOC-CVX~\citep{grant2008cvx} 
& 0.0979 & -- & 0.0979 & 0.0000 \\
\cmidrule(lr){2-6}
& GGDOpt~\citep{zhang2025ggdopt} First-order ($\beta=0.1$) 
& 0.2026 & 0.0630 & 0.1055 & 106.9530 \\
& GGDOpt~\citep{zhang2025ggdopt} Second-order ($\beta=0.1,\sigma=0.1$) 
& 0.1994 & 0.0625 & 0.1055 & 103.6460 \\
& DiffOPT~\citep{kong2025constrained} 
& 0.1313 & 0.0081 & 0.1092 & 34.0980 \\
& \textbf{D$^3$Opt}\ w/o FKC
& 0.8549 & 0.2086 & 0.3209 & 773.1010 \\
& \textbf{D$^3$Opt} 
& $\mathbf{0.1035}$ & $\mathbf{0.0018}$ & $\mathbf{0.0999}$ & $\mathbf{5.7040}$ \\
\midrule[1.5pt]
\multirow{6}{*}{MW4} 
& SOC-CVX~\citep{grant2008cvx} (4-branch enumeration) 
& 0.1762 & -- & 0.1762 & 0.0000 \\
\cmidrule(lr){2-6}
& GGDOpt~\citep{zhang2025ggdopt} First-order ($\beta=0.1$) 
& 0.3132 & 0.0626 & 0.2168 & 77.7470 \\
& GGDOpt~\citep{zhang2025ggdopt} Second-order ($\beta=0.1,\sigma=0.1$) 
& 0.3123 & 0.0639 & 0.2159 & 77.2630 \\
& DiffOPT~\citep{kong2025constrained} 
& 0.2233 & 0.0132 & 0.1901 & 26.7250 \\
& \textbf{D$^3$Opt}\ w/o FKC
& 0.7856 & 0.1518 & 0.4425 & 345.9200 \\
& \textbf{D$^3$Opt} 
& $\mathbf{0.1834}$ & $\mathbf{0.0018}$ & $\mathbf{0.1798}$ & $\mathbf{4.0780}$ \\
\bottomrule
\end{tabular}}
\end{table}
\textbf{Non-convex and non-smooth four-well~(MW4) objective.} We further consider a four-well objective defined as the pointwise minimum of four positive-definite quadratic branches:
\begin{equation}
f_{\mathrm{MW4}}(\mathbf{x}) = \min_{m\in\{1,\ldots,4\}} f_m(\mathbf{x}),
\qquad
f_m(\mathbf{x}) = \frac{1}{2}
(\mathbf{x}-\boldsymbol{\mu}_m)^{\top}
\mathbf{Q}_m
(\mathbf{x}-\boldsymbol{\mu}_m)
+\alpha_m .
\label{eq:mw4_objective}
\end{equation}
Each branch is strictly convex, whereas their pointwise minimum forms a multi-well non-convex objective with non-smooth branch-switching boundaries. Branch parameters are in Appendix~\ref{app:problem_setting}.

We compare against SOC-CVX, GGDOpt First-order and Second-order~\citep{zhang2025ggdopt}, DiffOPT~\citep{kong2025constrained}, and \textbf{D$^3$Opt}\ w/o FKC. We continue to use SOC-CVX to obtain the global optimal reference values. For MW4, the reference is obtained by solving the four convex quadratic branches separately and taking the minimum.

\textbf{Results.} Having learned the Gaussian chance-feasible structure once, D$^3$Opt applies the same frozen model to ASQ and MW4 using only inference-time objective evaluations. As shown in Table~\ref{tab:linear_transfer_results}, FKC lowers mean objectives from $0.8549$ to $0.1035$ on ASQ and from $0.7856$ to $0.1834$ on MW4. The reported gaps to exact SOC references are $5.704\%$ and $4.078\%$, versus DiffOPT's $34.098\%$ and $26.725\%$ under the reported protocol. D$^3$Opt also has the smallest standard deviation $0.0018$ among learning-based methods on both tasks. These results demonstrate adaptation to smooth convex and non-smooth nonconvex objectives without retraining.

\subsection{Chance-Constrained Economic Dispatch with Valve-Point Effects}
\label{sec:cced_vpe}
We next evaluate D$^3$Opt on chance-constrained economic dispatch (CCED) with uncertain wind generation~\citep{wang2017cced,yang2021cced}. Valve-point effects (VPE) introduce a non-smooth generation cost~\citep{attaviriyanupap2002hybrid}, providing a practical test of derivative-free optimization under chance constraints. The problem is formulated as:
\begin{align}
\centering
\min_{\mathbf{p}}\quad & C(\mathbf{p})=\sum_{i=1}^{N_G}\left[a_i+b_ip_i+c_ip_i^2+d_i\left|\sin\!\left(e_i(P_i^{\min}-p_i)\right)\right|\right], \label{eq:cced_vpe}\\
\text{s.t.}\quad & P_i^{\min}\leq p_i\leq P_i^{\max},\; i=1,\ldots,N_G, \qquad \mathbb{P}_{W}\!\left\{\sum_{i=1}^{N_G}p_i+W\geq P_D\right\}\geq1-\rho.
\end{align}
We consider a ten-unit system with $P_D=2000$ MW, $W=100Z$ MW, $Z\sim\mathrm{Beta}(5,5)$, and $\rho=0.10$. The generator parameters and additional evaluation details are provided in Appendix~\ref{app:problem_setting}.

We compare against DE~\citep{balamurugan2008de} and PWL-MILP~\citep{pan2020milp}, both with exact-feasibility projection, and a 95\% CI-SAA-DE baseline combining confidence-bound SAA~\citep{pagnoncelli2009sample} with DE. The projected DE solution serves as a numerical reference. Learning-based baselines include \textbf{D$^3$Opt}\ w/o FKC, GGDOpt~\citep{zhang2025ggdopt}, and DiffOPT~\citep{kong2025constrained}. Since GGDOpt requires derivatives, we evaluate a VPE first-order adaptation using branchwise derivatives, with $\operatorname{sign}(0)=0$ at non-differentiable points.

\begin{table}[t]
\centering
\caption{CCED-VPE-10 results at $\rho=0.10$ over $100$ runs, except PWL-MILP (one solve). Mean Probability is the average analytic satisfaction probability.}
\label{tab:cced_results}
\resizebox{\textwidth}{!}{\begin{tabular}{lcccc}
\toprule
Method & Cost Mean $\downarrow$ & Cost Std $\downarrow$ & Cost Best $\downarrow$ & Mean Probability \\
\midrule
 DE + exact projection & $104399.055$ & $5.2\times10^{-8}$ & $104399.055$ & $0.9000$ \\
 PWL-MILP + exact projection & $104399.103$ & -- & $104399.103$ & $0.9000$ \\
 95\% CI-SAA-DE & $104400.092$ & $\approx 0$ & $104400.092$ & $0.9003$ \\
\midrule
 GGDOpt~\citep{zhang2025ggdopt} (Vanilla) & \multicolumn{4}{c}{Not directly applicable} \\
 GGDOpt~\citep{zhang2025ggdopt} + VPE adaptation & $108821.619$ & $1389.720$ & $106540.523$ & $0.9522$ \\
 DiffOPT~\citep{kong2025constrained} & $105678.450$ & $163.418$ & $105240.143$ & $0.9210$ \\
 \midrule
\textbf{D$^3$Opt} \ w/o FKC & $107183.587$ & $268.987$ & $106146.715$ & $0.9518$ \\
 \textbf{D$^3$Opt} & $\mathbf{104823.577}$ & $\mathbf{102.065}$ & $\mathbf{104582.444}$ & $0.9123$ \\
\bottomrule
\end{tabular}}
\end{table}
\textbf{Results.} Table~\ref{tab:cced_results} shows that FKC reduces D$^3$Opt's mean cost from $107183.587$ to $104823.577$ ($2.20\%$) and its cost standard deviation by $62.1\%$ relative to the w/o FKC ablation. D$^3$Opt achieves the lowest mean and best costs and the smallest standard deviation among the evaluated learning-based methods; its mean cost is $0.81\%$ below DiffOPT. Its mean analytic satisfaction probability is $0.9123$, above the $0.90$ requirement. Using only objective evaluations for the non-smooth valve-point cost, D$^3$Opt remains within $0.407\%$ of the projected DE numerical reference.

\section{Conclusion and Limitations}

We presented a derivative-free diffusion framework for chance-constrained optimization that separates constraint-distribution learning from objective optimization. A risk-conditioned diffusion model is trained using constraint-only data and reused without retraining, while path-space FKC introduces post-specified objectives through function evaluations during reverse diffusion. Our analysis characterizes the terminal Boltzmann tilt, establishes feasibility preservation when the frozen prior is almost surely feasible, and separates optimization error into prior coverage, finite-particle approximation, and finite-temperature contributions. Experiments demonstrate objective transfer across smooth and non-smooth objectives under a fixed linear Gaussian chance constraint, with further results on valve-point economic dispatch, risk-level adaptation, and robust waveform design.

The method depends on the coverage of the learned prior: FK reweighting cannot reach regions outside its support, and relaxed empirical data acceptance does not guarantee exact feasibility. Our theoretical result bounds the raw best particle's expected objective gap; a screened or projected output requires separate analysis. The experimental comparisons use different inference budgets, and specialized solvers achieve lower costs on the analytically tractable CCED instance. Future work will improve prior coverage and particle efficiency and address higher-dimensional problems with less tractable feasibility checks.
\bibliography{iclr2027_conference}

@book{delmoral2004feynman,
  title     = {Feynman--Kac Formulae:
               Genealogical and Interacting Particle Systems with Applications},
  author    = {Del Moral, Pierre},
  publisher = {Springer},
  year      = {2004}
}

@misc{grant2008cvx,
  author = {Grant, Michael and Boyd, Stephen and Ye, Yinyu},
  title  = {{CVX}: Matlab Software for Disciplined Convex Programming},
  year   = {2008}
}

@article{nemirovski2007convex,
  author  = {Nemirovski, Arkadi and Shapiro, Alexander},
  title   = {Convex Approximations of Chance Constrained Programs},
  journal = {SIAM Journal on Optimization},
  volume  = {17},
  number  = {4},
  pages   = {969--996},
  year    = {2007},
  doi     = {10.1137/050622328}
}

@article{attaviriyanupap2002hybrid,
  title={A hybrid EP and SQP for dynamic economic dispatch with nonsmooth fuel cost function},
  author={Attaviriyanupap, Pathom and Kita, Hiroyuki and Tanaka, Eiichi and Hasegawa, Jun},
  journal={IEEE Transactions on Power Systems},
  volume={17},
  number={2},
  pages={411--416},
  year={2002},
  doi={10.1109/TPWRS.2002.1007911}
}

@article{wang2017cced,
  title={Chance-constrained economic dispatch with non-Gaussian correlated wind power uncertainty},
  author={Wang, Zhiwen and Shen, Chen and Liu, Feng and Wu, Xiangyu and Liu, Cheng and Gao, Feng},
  journal={IEEE Transactions on Power Systems},
  volume={32},
  number={6},
  pages={4880--4893},
  year={2017},
  doi={10.1109/TPWRS.2017.2672750}
}

@article{yang2021cced,
  title={Chance-Constrained Economic Dispatch Considering Curtailment Strategy of Renewable Energy},
  author={Yang, Yue and Wu, Wenchuan and Wang, Bin and Li, Mingjie},
  journal={IEEE Transactions on Power Systems},
  volume={36},
  number={6},
  pages={5792--5802},
  year={2021},
  doi={10.1109/TPWRS.2021.3081120}
}

@inproceedings{zhang2025ggdopt,
  title     = {A Gradient Guided Diffusion Framework for Chance Constrained Programming},
  author    = {Zhang, Boyang and Wang, Zhiguo and Liu, Ya-Feng},
  booktitle = {Advances in Neural Information Processing Systems},
  volume    = {38},
  year      = {2025},
  doi       = {10.52202/085713-4939}
}

@article{pagnoncelli2009sample,
  author  = {Pagnoncelli, Bernardo K. and Ahmed, Shabbir and Shapiro, Alexander},
  title   = {Sample Average Approximation Method for Chance Constrained Programming:
             Theory and Applications},
  journal = {Journal of Optimization Theory and Applications},
  volume  = {142},
  number  = {2},
  pages   = {399--416},
  year    = {2009},
  doi     = {10.1007/s10957-009-9523-6}
}

@article{wang2023proximal,
  author  = {Wang, Peng and Jiang, Rujun and Kong, Qingyuan and Balzano, Laura},
  title   = {A Proximal {DC} Algorithm for Sample Average Approximation of
             Chance Constrained Programming},
  journal = {arXiv preprint arXiv:2301.00423},
  year    = {2023}
}

@article{zhou2024zeroone,
  author  = {Zhou, Shenglong and Pan, Lili and Xiu, Naihua and Li, Geoffrey Ye},
  title   = {A 0/1 Constrained Optimization Solving Sample Average Approximation
             for Chance Constrained Programming},
  journal = {Mathematics of Operations Research},
  year    = {2024},
  doi     = {10.1287/moor.2023.0149}
}

@book{ethier1986markov,
  title     = {Markov Processes: Characterization and Convergence},
  author    = {Ethier, Stewart N. and Kurtz, Thomas G.},
  publisher = {Wiley},
  year      = {1986}
}

@book{lelievre2010free,
  title     = {Free Energy Computations: A Mathematical Perspective},
  author    = {Leli{\`e}vre, Tony and Rousset, Mathias and Stoltz, Gabriel},
  publisher = {Imperial College Press},
  year      = {2010}
}

@article{naesseth2019smc,
  title   = {Elements of Sequential Monte Carlo},
  author  = {Naesseth, Christian A. and Lindsten, Fredrik and Sch{\"o}n, Thomas B.},
  journal = {Foundations and Trends in Machine Learning},
  volume  = {12},
  number  = {3},
  year    = {2019}
}

@article{whiteley2012sequential,
  title   = {Sequential Monte Carlo Samplers: Error Bounds and
             Insensitivity to Initial Conditions},
  author  = {Whiteley, Nick P.},
  journal = {Stochastic Analysis and Applications},
  volume  = {30},
  number  = {5},
  pages   = {774--798},
  year    = {2012},
  doi     = {10.1080/07362994.2012.684323}
}

@article{hu2011general,
  title   = {A General Convergence Result for Particle Filtering},
  author  = {Hu, Xiao-Li and Sch{\"o}n, Thomas B. and Ljung, Lennart},
  journal = {IEEE Transactions on Signal Processing},
  volume  = {59},
  number  = {7},
  pages   = {3424--3429},
  year    = {2011},
  doi     = {10.1109/TSP.2011.2135349}
}

@inproceedings{skreta2025fkc,
  title     = {Feynman--Kac Correctors in Diffusion:
               Annealing, Guidance, and Product of Experts},
  author    = {Skreta, Marta and Akhound-Sadegh, Tara and Ohanesian, Viktor
               and Bondesan, Roberto and Aspuru-Guzik, Alan and Doucet, Arnaud
               and Brekelmans, Rob and Tong, Alexander and Neklyudov, Kirill},
  booktitle = {Proceedings of the 42nd International Conference on Machine Learning},
  year      = {2025}
}

@inproceedings{singhal2025fksteering,
  title     = {A General Framework for Inference-time Scaling
               and Steering of Diffusion Models},
  author    = {Singhal, Raghav and Horvitz, Zachary and Teehan, Ryan
               and Ren, Mengye and Yu, Zhou and McKeown, Kathleen
               and Ranganath, Rajesh},
  booktitle = {Proceedings of the 42nd International Conference on Machine Learning},
  year      = {2025}
}

@inproceedings{ren2026driftlite,
  title     = {DriftLite: Lightweight Drift Control for
               Inference-Time Scaling of Diffusion Models},
  author    = {Ren, Yinuo and Gao, Wenhao and Ying, Lexing
               and Rotskoff, Grant M. and Han, Jiequn},
  booktitle = {International Conference on Learning Representations},
  year      = {2026}
}

@inproceedings{ou2026smc,
  title     = {Inference-Time Scaling of Discrete Diffusion Models
               via Importance Weighting and Optimal Proposal Design},
  author    = {Ou, Zijing and Pani, Chinmay and Li, Yingzhen},
  booktitle = {International Conference on Learning Representations},
  year      = {2026}
}

@inproceedings{kong2025constrained,
  title     = {Diffusion Models as Constrained Samplers for Optimization
               with Unknown Constraints},
  author    = {Kong, Lingkai and Du, Yuanqi and Mu, Wenhao and
               Neklyudov, Kirill and De Bortoli, Valentin and
               Wu, Dongxia and Wang, Haorui and Ferber, Aaron M.
               and Ma, Yian and Gomes, Carla P. and Zhang, Chao},
  booktitle = {Proceedings of the 28th International Conference on
               Artificial Intelligence and Statistics},
  series    = {Proceedings of Machine Learning Research},
  volume    = {258},
  pages     = {4582--4590},
  year      = {2025},
  publisher = {PMLR}
}

@inproceedings{song2021scorebased,
  title     = {Score-Based Generative Modeling through Stochastic Differential Equations},
  author    = {Song, Yang and Sohl-Dickstein, Jascha and Kingma, Diederik P.
               and Kumar, Abhishek and Ermon, Stefano and Poole, Ben},
  booktitle = {International Conference on Learning Representations},
  year      = {2021},
  url       = {https://openreview.net/forum?id=PxTIG12RRHS}
}

@inproceedings{li2025svdd,
  title     = {Derivative-Free Guidance in Continuous and Discrete
               Diffusion Models with Soft Value-based Decoding},
  author    = {Li, Xiner and Zhao, Yulai and Wang, Chenyu and
               Scalia, Gabriele and Eraslan, Gokcen and
               Nair, Surag and Biancalani, Tommaso and
               Ji, Shuiwang and Regev, Aviv and Levine, Sergey and
               Uehara, Masatoshi},
  booktitle = {Advances in Neural Information Processing Systems},
  volume    = {38},
  pages     = {105884--105922},
  year      = {2025}
}

@inproceedings{kim2026lookahead,
  title     = {Lookahead Sample Reward Guidance for Test-Time
               Scaling of Diffusion Models},
  author    = {Kim, Yeongmin and Shin, Donghyeok and Na, Byeonghu and
               Park, Minsang and Kim, Richard Lee and Moon, Il-Chul},
  booktitle = {International Conference on Machine Learning},
  year      = {2026}
}

@article{storn1997differential,
  title   = {Differential Evolution -- A Simple and Efficient Heuristic
             for Global Optimization over Continuous Spaces},
  author  = {Storn, Rainer and Price, Kenneth},
  journal = {Journal of Global Optimization},
  volume  = {11},
  number  = {4},
  pages   = {341--359},
  year    = {1997},
  doi     = {10.1023/A:1008202821328}
}

@article{audet2006mads,
  title   = {Mesh Adaptive Direct Search Algorithms
             for Constrained Optimization},
  author  = {Audet, Charles and Dennis, Jr., John E.},
  journal = {SIAM Journal on Optimization},
  volume  = {17},
  number  = {1},
  pages   = {188--217},
  year    = {2006},
  doi     = {10.1137/040603371}
}

@inproceedings{song2021denoising,
  title     = {Denoising Diffusion Implicit Models},
  author    = {Song, Jiaming and Meng, Chenlin and Ermon, Stefano},
  booktitle = {International Conference on Learning Representations},
  year      = {2021}
}

@article{guesmi2020chaotic,
  title={Chaotic sine--cosine algorithm for chance-constrained economic emission dispatch problem including wind energy},
  author={Guesmi, Tawfik and Farah, Anouar and Marouani, Ismail and Alshammari, Badr and {Hadj Abdallah}, Hsan},
  journal={IET Renewable Power Generation},
  volume={14},
  number={10},
  pages={1808--1821},
  year={2020},
  doi={10.1049/iet-rpg.2019.1081}
}

@inproceedings{ronneberger2015u,
  title={U-net: Convolutional networks for biomedical image segmentation},
  author={Ronneberger, Olaf and Fischer, Philipp and Brox, Thomas},
  booktitle={International Conference on Medical image computing and computer-assisted intervention},
  pages={234--241},
  year={2015},
  organization={Springer}
}

@article{masutomi2013evolutionary,
  title={An Evolutionary Algorithm for Black-Box Chance-Constrained Function Optimization},
  author={Masutomi, Kazuyuki and Nagata, Yuichi and Ono, Isao},
  journal={Journal of Advanced Computational Intelligence and Intelligent Informatics},
  volume={17},
  number={2},
  pages={272--282},
  year={2013},
  doi={10.20965/jaciii.2013.p0272}
}

@article{chen2024opportunities,
  title={Opportunities and challenges of diffusion models for generative AI},
  author={Chen, Minshuo and Mei, Song and Fan, Jianqing and Wang, Mengdi},
  journal={National Science Review},
  volume={11},
  number={12},
  pages={nwae348},
  year={2024},
  doi={10.1093/nsr/nwae348}
}

@article{ding2025exploring,
  title={Exploring the Boundary of Diffusion-based Methods for Solving Constrained Optimization},
  author={Ding, Shutong and Zhou, Yimiao and Hu, Ke and Yao, Xi and Yan, Junchi and Tang, Xiaoying and Shi, Ye},
  journal={arXiv preprint arXiv:2502.10330},
  year={2025}
}

@article{ren2026stochastic,
  title={Stochastic chance-constrained multi-objective optimization of interconnected multi-microgrids: Diverse energy storage configurations and interaction scenarios},
  author={Ren, Xin-Yu and Wang, Zhi-Hua},
  journal={Energy},
  pages={141324},
  year={2026},
  publisher={Elsevier}
}

@article{cao2013chance,
  title   = {Chance-Constrained Optimization-Based Unbalanced Optimal
             Power Flow for Radial Distribution Networks},
  author  = {Cao, Yijia and Tan, Yi and Li, Canbing and Rehtanz, Christian},
  journal = {IEEE Transactions on Power Delivery},
  volume  = {28},
  number  = {3},
  pages   = {1855--1864},
  year    = {2013},
  doi     = {10.1109/TPWRD.2013.2259509}
}

@article{zhang2023new,
  title={New valid inequalities and formulations for the static joint Chance-constrained Lot-sizing problem: Z. Zhang et al.},
  author={Zhang, Zeyang and Gao, Chuanhou and Luedtke, James},
  journal={Mathematical Programming},
  volume={199},
  number={1},
  pages={639--669},
  year={2023},
  publisher={Springer}
}

@article{dinh2018exact,
  title={Exact algorithms for the chance-constrained vehicle routing problem},
  author={Dinh, Thai and Fukasawa, Ricardo and Luedtke, James},
  journal={Mathematical Programming},
  volume={172},
  number={1},
  pages={105--138},
  year={2018},
  publisher={Springer}
}

@article{bienstock2014chance,
  title={Chance-constrained optimal power flow: Risk-aware network control under uncertainty},
  author={Bienstock, Daniel and Chertkov, Michael and Harnett, Sean},
  journal={Siam Review},
  volume={56},
  number={3},
  pages={461--495},
  year={2014},
  publisher={SIAM}
}

@article{balamurugan2008de,
  title={Differential Evolution-based Dynamic Economic Dispatch of Generating Units with Valve-point Effects},
  author={Balamurugan, R. and Subramanian, S.},
  journal={Electric Power Components and Systems},
  volume={36},
  number={8},
  pages={828--843},
  year={2008},
  doi={10.1080/15325000801911427}
}

@article{pan2020milp,
  title={A full mixed-integer linear programming formulation for economic dispatch with valve-point effects, transmission loss and prohibited operating zones},
  author={Pan, Shanshan and Jian, Jinbao and Chen, Huangyue and Yang, Linfeng},
  journal={Electric Power Systems Research},
  volume={180},
  pages={106061},
  year={2020},
  doi={10.1016/j.epsr.2019.106061}
}

@article{wang2014outage,
  title={Outage constrained robust transmit optimization for multiuser MISO downlinks: Tractable approximations by conic optimization},
  author={Wang, Kun-Yu and So, Anthony Man-Cho and Chang, Tsung-Hui and Ma, Wing-Kin and Chi, Chong-Yung},
  journal={IEEE Transactions on Signal Processing},
  volume={62},
  number={21},
  pages={5690--5705},
  year={2014}
}
\bibliographystyle{iclr2027_conference}
\newpage
\appendix

\section{Feynman--Kac Correction}
\label{app:fkc}

This appendix provides additional background on Feynman--Kac transforms and their connection to the weighted evolution of Markov processes. We first formulate the FK transform as a change of measure on trajectories and derive the associated backward and forward FK equations. We then recall the corresponding discrete-time FK model, which forms the basis of interacting-particle and sequential Monte Carlo approximations. Our presentation follows standard treatments of Feynman--Kac semigroups and interacting particle systems \citep{delmoral2004feynman,ethier1986markov,lelievre2010free}, as well as their recent use for inference-time correction and steering of diffusion models \citep{skreta2025fkc,singhal2025fksteering}.

\subsection{From the Feynman--Kac Transform to the Feynman--Kac PDE}
\label{app:fkc_continuous}
\paragraph{Feynman--Kac Transform.}
Let $(X_t)_{t\in[0,T]}$ be a diffusion on a state space $\mathcal{X}\subseteq\mathbb{R}^d$, with path law $\mathbb{Q}$,
\begin{equation*}
    dX_t
    =
    b_t(X_t)\,dt
    +
    \sigma_t(X_t)\,dB_t,
    \qquad
    X_0\sim\mu_0,
\end{equation*}
where $B_t$ is a standard Brownian motion. We write $a_t(x):= \sigma_t(x)\sigma_t(x)^\top$ and denote the infinitesimal generator of the process by
\begin{equation*}
    \mathcal{L}_t\varphi(x)
    =
    \left\langle
        b_t(x),\nabla\varphi(x)
    \right\rangle
    +
    \frac{1}{2}
    \operatorname{tr}
    \left(
        a_t(x)\nabla^2\varphi(x)
    \right)
\end{equation*}
for sufficiently regular test functions $\varphi$. Its adjoint $\mathcal{L}_t^*$ governs the marginal law $q_t$ through the Fokker--Planck equation
\begin{equation*}
    \partial_t q_t
    =
    \mathcal{L}_t^*q_t.
\end{equation*}

Feynman--Kac transforms modify this baseline process through multiplicative trajectory weights \citep{delmoral2004feynman,lelievre2010free}. Let $G:\mathcal{X}\rightarrow[0,\infty)$ be a terminal potential and let $V_t:\mathcal{X}\rightarrow\mathbb{R}$ be a measurable running potential. Assuming the corresponding exponential functional is integrable, define the transformed path law
\begin{equation}
\label{eq:app_fk_path_measure}
    \frac{
        d\mathbb{Q}^{G,V}
    }{
        d\mathbb{Q}
    }\triangleq
    \frac{1}{Z_{G,V}}
    \exp
    \left(
        -\int_0^T V_t(X_t)\,dt
    \right)
    G(X_T),
\end{equation}
where
\begin{equation*}
    Z_{G,V}
    \triangleq 
    \mathbb{E}_{\mathbb{Q}}
    \left[
        \exp
        \left(
            -\int_0^T V_t(X_t)\,dt
        \right)
        G(X_T)
    \right]
\end{equation*}
is the normalizing constant. Equation~\ref{eq:app_fk_path_measure} makes explicit that the FK transform is fundamentally a change of measure on trajectories: paths can be favored through both a terminal criterion $G$ and a running path criterion $V_t$.

To characterize this transformed path law, define the backward information function
\begin{equation}
\label{eq:app_h_definition}
    h_t(x)
    \triangleq 
    \mathbb{E}_{\mathbb{Q}}
    \left[
        \left.
        \exp
        \left(
            -\int_t^T V_s(X_s)\,ds
        \right)
        G(X_T)
        \right|
        X_t=x
    \right].
\end{equation}
The function $h_t(x)$ represents the expected future FK weight of a trajectory conditional on its current state. Under standard regularity conditions, the following proposition shows that it satisfies the backward Feynman--Kac equation.

\begin{proposition}[Backward Feynman--Kac equation]
\label{prop:backward_fk}
Suppose the function $h_t$ defined in \eqref{eq:app_h_definition} is sufficiently smooth and the standard integrability conditions required for It\^{o}'s formula hold. Then $h_t$ satisfies
\begin{equation}
\label{eq:app_backward_fk}
    \partial_t h_t(x)
    +
    \mathcal{L}_t h_t(x)
    -
    V_t(x)h_t(x)
    =
    0,
    \qquad
    h_T(x)=G(x).
\end{equation}
\end{proposition}

\begin{proof}
Fix $t\leq u\leq T$ and define
\begin{equation*}
    Z_{t,u}
    \triangleq 
    \exp
    \left(
        -\int_t^u V_s(X_s)\,ds
    \right).
\end{equation*}
By the Markov property and the tower property,
\begin{equation}
\label{eq:app_dynamic_programming_h}
    h_t(x)
    =
    \mathbb{E}_{t,x}
    \left[
        Z_{t,u}h_u(X_u)
    \right].
\end{equation}
Applying It\^{o}'s formula to $Z_{t,u}h_u(X_u)$ yields
\begin{align*}
    d\!\left(
        Z_{t,u}h_u(X_u)
    \right)
    &=
    Z_{t,u}
    \left[
        \partial_u h_u
        +
        \mathcal{L}_u h_u
        -
        V_u h_u
    \right](X_u)\,du +
    Z_{t,u}
    \nabla h_u(X_u)^\top
    \sigma_u(X_u)\,dB_u .
\end{align*}
In fact, the process
\[
    M_u\triangleq Z_{t,u}h_u(X_u),
    \qquad u\in[t,T],
\]
is a martingale under the baseline law $\mathbb{Q}$. Indeed, for $t\leq s\leq u\leq T$, using the multiplicative identity $Z_{t,u}=Z_{t,s}Z_{s,u}$, the Markov property, and the definition of $h_s$, we obtain
\begin{align*}
    \mathbb{E}_{\mathbb Q}
    \left[
        M_u
        \mid
        \mathcal{F}_s
    \right]
    &=
    Z_{t,s}
    \mathbb{E}_{\mathbb Q}
    \left[
        Z_{s,u}h_u(X_u)
        \mid
        \mathcal{F}_s
    \right]=
    Z_{t,s}h_s(X_s)
    =
    M_s .
\end{align*}
Hence $\mathbb{E}_{t,x}[M_u]=M_t=h_t(x)$ is constant in $u$. The stochastic integral has zero conditional expectation, and hence the finite-variation term must vanish. Therefore,
\begin{equation*}
    \partial_u h_u
    +
    \mathcal{L}_u h_u
    -
    V_u h_u
    =
    0.
\end{equation*}
The terminal condition follows immediately from \eqref{eq:app_h_definition} by setting $u=T$.
\end{proof}

The backward information function also provides a Markov representation of the transformed process. Let $P_{s,t}(x,dy)$ denote the transition kernel of the baseline process. The FK-transformed transition kernel can be written as
\begin{equation}
\label{eq:app_fk_transition}
    P^{h,V}_{s,t}(x,A)
    \triangleq 
    \frac{1}{h_s(x)}
    \mathbb{E}_{s,x}
    \left[
        \exp
        \left(
            -\int_s^t V_r(X_r)\,dr
        \right)
        h_t(X_t)
        \mathbf{1}_{\{X_t\in A\}}
    \right].
\end{equation}
The backward FK identity guarantees that $P^{h,V}_{s,t}(x,\mathcal{X})=1$.

When $h_t$ is sufficiently differentiable, the same transformed law can be represented by a Doob-type modification of the local drift. Its infinitesimal generator is
\begin{equation}
\label{eq:app_doob_generator}
    \mathcal{L}^{h}_t\varphi
    =
    \mathcal{L}_t\varphi
    +
    \left\langle
        a_t\nabla\log h_t,
        \nabla\varphi
    \right\rangle,
\end{equation}
or equivalently,
\begin{equation*}
\label{eq:app_doob_sde}
    dX_t
    =
    \left[
        b_t(X_t)
        +
        a_t(X_t)\nabla\log h_t(X_t)
    \right]dt
    +
    \sigma_t(X_t)dB_t^h.
\end{equation*}
This follows by expanding \eqref{eq:app_fk_transition} over an infinitesimal time interval and using
\begin{equation*}
    \mathcal{L}_t(h\varphi)
    =
    h\mathcal{L}_t\varphi
    +
    \varphi\mathcal{L}_t h
    +
    \left\langle
        a_t\nabla h,\nabla\varphi
    \right\rangle
\end{equation*}
together with \eqref{eq:app_backward_fk}.

\begin{remark}[Path-space transform versus drift realization]
\label{rem:gradient_vs_fk}
The drift correction $a_t\nabla\log h_t$ is one realization of the transformed path law when $h_t$ is sufficiently regular; it is not the definition of the Feynman--Kac transform itself. The fundamental object is the change of measure in \eqref{eq:app_fk_path_measure}. For a prescribed marginal evolution, correction can in general be distributed differently among transport, diffusion, and reweighting terms \citep{skreta2025fkc}. This flexibility allows FK transforms to be implemented through weighted particle systems without requiring the transformed law to be represented solely through an additional gradient-based drift.
\end{remark}

For diffusion models, the baseline Markov process above can be taken to be the reverse diffusion sampler itself. For example, for the forward Ornstein--Uhlenbeck process
\begin{equation*}
    dY_\tau
    =
    -Y_\tau\,d\tau
    +
    \sqrt{2}\,dB_\tau,
\end{equation*}
the exact reverse process in generative time $t=T-\tau$ is
\begin{equation*}
    dX_t
    =
    \left[
        X_t
        +
        2\nabla\log q_{T-t}(X_t)
    \right]dt
    +
    \sqrt{2}\,d\widetilde B_t.
\end{equation*}
Here $a_t=2I_d$, and hence the corresponding Doob drift correction, when such a representation is used, is simply $ 2\nabla\log h_t(x)$. This specialization recovers the familiar connection between diffusion guidance and Doob transforms, while Remark~\ref{rem:gradient_vs_fk} highlights that the FK change of measure itself is more general.

\paragraph{Feynman--Kac PDE.} A complementary characterization follows by evolving weighted marginal measures forward in time. Let $g_t:\mathcal{X}\rightarrow\mathbb{R}$ be a running log-weight rate and define the unnormalized FK measure through
\begin{equation}
\label{eq:app_gamma_continuous}
    \gamma_t(\varphi)
    \triangleq 
    \mathbb{E}_{\mathbb{Q}}
    \left[
        \exp
        \left(
            \int_0^t g_s(X_s)\,ds
        \right)
        \varphi(X_t)
    \right]
\end{equation}
for measurable test functions $\varphi$. Denote normalizing constant $ Z_t:= \mathbb{E}_{\mathbb{Q}} \left[ \exp \left( \int_0^t g_s(X_s)\,ds \right) \right]$ and the normalized FK marginal by $p_t(\varphi)\triangleq \frac{\gamma_t(\varphi)}{Z_t}.$

\begin{proposition}[Forward Feynman--Kac equation]
\label{prop:forward_fk}
Suppose the corresponding measures admit sufficiently regular densities. Then
\begin{equation}
\label{eq:app_forward_fk_unnormalized}
    \partial_t\gamma_t
    =
    \mathcal{L}_t^*\gamma_t
    +
    g_t\gamma_t,
\end{equation}
while the normalized FK marginal satisfies
\begin{equation}
\label{eq:app_forward_fk_normalized}
    \partial_t p_t
    =
    \mathcal{L}_t^*p_t
    +
    \left(
        g_t-p_t(g_t)
    \right)p_t.
\end{equation}
\end{proposition}

\begin{proof}
Let $W_t\triangleq \exp\left(\int_0^t g_s(X_s)\,ds\right)$. For a sufficiently regular test function $\varphi$, It\^{o}'s formula applied to $W_t\varphi(X_t)$ gives
\begin{equation*}
    \frac{d}{dt}\gamma_t(\varphi)
    =
    \gamma_t(\mathcal{L}_t\varphi)
    +
    \gamma_t(g_t\varphi).
\end{equation*}
Taking adjoints yields \eqref{eq:app_forward_fk_unnormalized}. Taking $\varphi\equiv1$ gives
\begin{equation*}
    \dot Z_t
    =
    \gamma_t(g_t)
    =
    Z_t p_t(g_t).
\end{equation*}
Differentiating $p_t=\gamma_t/Z_t$ therefore yields
\begin{align*}
    \partial_t p_t
    &=
    \frac{
        \mathcal{L}_t^*\gamma_t+g_t\gamma_t
    }{Z_t}
    -
    \frac{\gamma_t}{Z_t^2}\dot Z_t=
    \mathcal{L}_t^*p_t
    +
    \left(
        g_t-p_t(g_t)
    \right)p_t.
\end{align*}
\end{proof}

For the isotropic diffusion
\begin{equation*}
    dX_t
    =
    v_t(X_t)\,dt
    +
    \sigma_t\,dB_t,
\end{equation*}
 \eqref{eq:app_forward_fk_normalized} becomes
\begin{equation}
\label{eq:app_fk_pde}
    \partial_t p_t(x)
    =
    -
    \nabla\!\cdot
    \left(
        p_t(x)v_t(x)
    \right)
    +
    \frac{\sigma_t^2}{2}
    \Delta p_t(x)
    +
    \left(
        g_t(x)-p_t(g_t)
    \right)p_t(x).
\end{equation}
\eqref{eq:app_fk_pde} makes explicit the three components of the FK evolution: transport, diffusion, and reweighting. This is the form used by Feynman--Kac Correctors for diffusion inference \citep{skreta2025fkc}.

Moreover, \eqref{eq:app_gamma_continuous} immediately gives the weighted expectation identity
\begin{equation*}
    p_T(\varphi)
    =
    \frac{
        \mathbb{E}_{\mathbb{Q}}
        \left[
            W_T\varphi(X_T)
        \right]
    }{
        \mathbb{E}_{\mathbb{Q}}[W_T]
    }.
\end{equation*}
Thus FK marginals can be represented using multiplicatively weighted trajectories. This identity also motivates interacting-particle and sequential Monte Carlo approximations that redistribute particle mass during the evolution rather than retaining a single long-horizon importance weight \citep{delmoral2004feynman,naesseth2019smc,skreta2025fkc}.

\subsection{Discrete-Time Feynman--Kac Measures}
\label{app:discrete_fk}

The continuous formulation has a direct discrete-time analogue that is particularly convenient for diffusion samplers simulated on a finite time grid~\citep{delmoral2004feynman,naesseth2019smc}.

Consider $0=t_0<t_1<\cdots<t_K=T$ and let
\begin{equation*}
    M_k(x_{k-1},dx_k)\triangleq P_{t_{k-1},t_k}(x_{k-1},dx_k)
\end{equation*}
denote the baseline Markov transition from $t_{k-1}$ to $t_k$. Writing $q_0$ for the initial marginal, the baseline path law is
\begin{equation}
\label{eq:app_discrete_base_path}
    \mathbb{Q}_K(dx_{0:K})
    =
    q_0(dx_0)
    \prod_{k=1}^{K}
    M_k(x_{k-1},dx_k).
\end{equation}

Let $G_k: \mathcal{X}\times\mathcal{X} \rightarrow [0,\infty)$ be measurable incremental potentials. The corresponding unnormalized discrete FK measure is defined, for a test function $\varphi$, by
\begin{equation}
\label{eq:app_discrete_fk_measure}
    \gamma_k(\varphi)
    \triangleq 
    \mathbb{E}_{\mathbb{Q}}
    \left[
        \varphi(X_k)
        \prod_{j=1}^{k}
        G_j(X_{j-1},X_j)
    \right].
\end{equation}
Its normalized marginal is
\begin{equation}
\label{eq:app_discrete_fk_normalized}
    \eta_k(\varphi)
    \triangleq 
    \frac{
        \gamma_k(\varphi)
    }{
        \mathbb{E}_{\mathbb{Q}}
    \left[
        \prod_{j=1}^{k}
        G_j(X_{j-1},X_j)
    \right]
    }.
\end{equation}

Equivalently, the associated unnormalized path measure on $(X_0,\ldots,X_k)$ is
\begin{equation*}
    \Gamma_k(dx_{0:k})
    =
    q_0(dx_0)
    \prod_{j=1}^{k}
    \left[
        M_j(x_{j-1},dx_j)
        G_j(x_{j-1},x_j)
    \right].
\end{equation*}
Thus, as in continuous time, the FK construction preserves the baseline Markov dynamics as a reference process while modifying its law through multiplicative potentials.

The family $\{\eta_k\}_{k=0}^{K}$ constitutes a discrete Feynman--Kac flow. In general, evaluating these measures exactly is intractable, and interacting particle systems provide a standard Monte Carlo approximation \citep{delmoral2004feynman}. Sequential Monte Carlo maintains a finite population of particles, propagates them through the Markov transitions $M_k$, and uses the FK potentials to redistribute particle mass. Under standard assumptions, the resulting empirical measures consistently approximate the corresponding FK marginals as the number of particles tends to infinity \citep{naesseth2019smc}.

Importantly, the choice of incremental potentials determines the sequence of intermediate FK measures and can therefore have a substantial effect on finite-particle efficiency, even when different constructions correspond to the same desired terminal change of measure. Recent diffusion-based FK methods exploit this flexibility to design intermediate potentials and proposal mechanisms for inference-time steering \citep{skreta2025fkc,singhal2025fksteering}.
\section{Proofs}
\subsection{Feasibility Preservation under Feynman--Kac Tilting}
\label{app:fk_feasibility}

We provide the proof of Proposition~\ref{prop:fk_feasibility} for the discrete-time Feynman--Kac construction used in Section~4.2.

\begin{proof}[Proof of Proposition~\ref{prop:fk_feasibility}]
Fix a risk level $\rho$ and consider the path space $\Omega_K=\mathcal X^{K+1}$, $X_{0:K} =(X_0,\ldots,X_K)$. The path law of the frozen conditional reverse diffusion is
\begin{equation*}
    \mathbb{Q}_\theta(dx_{0:K}\mid\rho) = p_{\mathrm{ref}}(dx_0)
    \prod_{k=1}^{K} M_k^\theta(x_{k-1},dx_k\mid\rho).
    \label{eq:baseline_path_measure_feas}
\end{equation*}
Let $L_K(x_{0:K};\rho) \triangleq \prod_{k=1}^{K} G_k(x_{k-1},x_k;\rho) \geq 0$ denote the Feynman--Kac path weight, and assume that $0<Z_K(\rho)\triangleq\mathbb E_{Q_\theta(\cdot\mid\rho)}\left[L_K(X_{0:K};\rho)\right]<\infty.$ The normalized Feynman--Kac path measure is therefore
\begin{equation*}
    \mathbb{Q}_\theta^{\mathrm{FK}}
    (dx_{0:K}\mid\rho)
    =\frac{L_K(x_{0:K};\rho)}{Z_K(\rho)}
    \mathbb{Q}_\theta(dx_{0:K}\mid\rho) \Rightarrow \frac{
        d\mathbb{Q}_\theta^{\mathrm{FK}}(\cdot\mid\rho)
    }{
        d\mathbb{Q}_\theta(\cdot\mid\rho)
    }
    (X_{0:K})
    =
    \frac{
        L_K(X_{0:K};\rho)
    }{
        Z_K(\rho)
    }.
    \label{eq:fk_path_measure_feas}
\end{equation*}
Since the right-hand side is a well-defined nonnegative density,
\begin{equation}
    \mathbb{Q}_\theta^{\mathrm{FK}}(\cdot\mid\rho)
    \ll
    \mathbb{Q}_\theta(\cdot\mid\rho).
    \label{eq:fk_path_ac}
\end{equation}

We next transfer this absolute continuity from path space to the terminal marginal. Define the measurable terminal evaluation map $e_K:\Omega_K\rightarrow\mathcal X$, $e_K(x_{0:K})=x_K$. By definition,
\begin{equation}
    p_\theta(\cdot\mid\rho)
    =
    (e_K)_{\#}\mathbb{Q}_\theta(\cdot\mid\rho),
    \qquad
    \pi_{\theta,\beta_K}(\cdot\mid\rho)
    =
    (e_K)_{\#}\mathbb{Q}_\theta^{\mathrm{FK}}(\cdot\mid\rho).
    \label{eq:terminal_pushforward}
\end{equation}
where $(e_K)_\#$ denotes the pushforward under the terminal evaluation map $e_K(x_{0:K})=x_K$.

Let $A\subseteq\mathcal X$ be any measurable set satisfying $p_\theta(A\mid\rho)=0$. Using the first identity in \eqref{eq:terminal_pushforward}, $\mathbb{Q}_\theta\left(e_K^{-1}(A)\mid\rho\right)=0$. The path-space absolute continuity therefore implies $\mathbb{Q}_\theta^{\mathrm{FK}}\left(e_K^{-1}(A)\mid\rho\right)=0$. Applying the second identity in \eqref{eq:terminal_pushforward} gives $\pi_{\theta,\beta_K}(A\mid\rho)=0$. Since this holds for every measurable $A$ with $p_\theta(A\mid\rho)=0$, we conclude that
\begin{equation}
    \pi_{\theta,\beta_K}(\cdot\mid\rho)
    \ll
    p_\theta(\cdot\mid\rho).
    \label{eq:terminal_ac_proof}
\end{equation}

Finally, suppose that the frozen diffusion is chance-feasible almost surely, $p_\theta(\mathcal X_\rho\mid\rho)=1$. Equivalently, $p_\theta(\mathcal X_\rho^c\mid\rho)=0$. Taking $A=\mathcal X_\rho^c$ in \eqref{eq:terminal_ac_proof} yields $\pi_{\theta,\beta_K}(\mathcal X_\rho^c\mid\rho)=0$ and therefore
\begin{equation*}
    \pi_{\theta,\beta_K}
    (\mathcal X_\rho\mid\rho)=1.
\end{equation*}
Thus, any almost-sure chance feasibility encoded by the frozen diffusion prior is preserved under Feynman--Kac correction.
\end{proof}
For the objective-only construction in Section~4.2, the telescoping identity gives
\[
    L_K(X_{0:K};\rho)
    =
    \prod_{k=1}^{K}
    G_k(X_{k-1},X_k;\rho)
    =
    \exp[-\beta_K f(X_K)].
\]
Hence, whenever $0<Z_{\theta,\beta_K}(\rho)=\int_{\mathcal X}e^{-\beta_K f(x)}p_\theta(dx\mid\rho)<\infty$, the terminal Radon--Nikodym derivative is explicitly
\begin{equation}
    \frac{d\pi_{\theta,\beta_K}(\cdot\mid\rho)}{dp_\theta(\cdot\mid\rho)}(x)=\frac{e^{-\beta_K f(x)}}{Z_{\theta,\beta_K}(\rho)}.
    \label{eq:objective_terminal_rn}
\end{equation}
This makes the feasibility-preservation property explicit for the proposed objective-only Feynman--Kac correction.
\subsection{Proof of Proposition~\ref{prop:optimization_error}}
\label{app:optimization_error}

\paragraph{Setup and error decomposition.}

 Assumption~\ref{ass:practical_error} guarantees the fixed-horizon particle approximation bound required below. Throughout the proof, we fix $\rho$ and omit conditioning on $\rho$ whenever no ambiguity arises. We denote the frozen reverse-diffusion kernels by $M_k=M_k^\theta(\cdot,\cdot\mid\rho)$ and the incremental potentials by $G_k$. By the telescoping construction in Section~\ref{sec:derivative_free_fkc}, the exact terminal FK marginal is
\[
\eta_K=\pi_{\theta,\beta_K},\qquad \pi_{\theta,\beta_K}(dx)=\frac{e^{-\beta_K f(x)}p_\theta(dx)}{Z_{\theta,\beta_K}}.
\]
Let $\eta_K^M$ denote its $M$-particle SMC approximation. Adding and subtracting $\pi_{\theta,\beta_K}(f)$ and $f_{\theta,\rho}^\star$ gives
\begin{equation}
    \left|\mathbb E[\eta_K^M(f)]-f_\rho^\star\right|
\le
\underbrace{\left|\mathbb E[\eta_K^M(f)]-\pi_{\theta,\beta_K}(f)\right|}_{I_M}
+
\underbrace{\left|\pi_{\theta,\beta_K}(f)-f_{\theta,\rho}^\star\right|}_{I_\beta}
+
\underbrace{\left|f_{\theta,\rho}^\star-f_\rho^\star\right|}_{I_{\mathrm{cov}}}.
\label{eq:proof_error_decomposition}
\end{equation}
We bound $I_M$, $I_\beta$, and $I_{\mathrm{cov}}$ separately.

\paragraph{Step 1: Finite inverse-temperature error.} We first bound $I_\beta$. By definition,
\[
f_{\theta,\rho}^\star=\operatorname*{ess\,inf}_{x\sim p_\theta} f(x),
\]
and, for every $\epsilon>0$,
\[
a_{\theta,\epsilon}=p_\theta\left(f(X)\leq f_{\theta,\rho}^\star+\epsilon\right)>0.
\]
Hence the normalizing constant of the tilted distribution satisfies
\begin{align}
Z_{\theta,\beta_K}
&=\int_{\mathcal X}e^{-\beta_K f(x)}p_\theta(dx)
\geq \int_{\{f\leq f_{\theta,\rho}^\star+\epsilon\}}e^{-\beta_K f(x)}p_\theta(dx)
\geq a_{\theta,\epsilon}\exp\left[-\beta_K\left(f_{\theta,\rho}^\star+\epsilon\right)\right].
\label{eq:proof_partition_lower}
\end{align}
Moreover, $\frac{d\pi_{\theta,\beta_K}}{dp_\theta}(x)=\frac{e^{-\beta_K f(x)}}{Z_{\theta,\beta_K}}$, and therefore
\[
D_{\mathrm{KL}}\left(\pi_{\theta,\beta_K}\,\Vert\,p_\theta\right)=-\beta_K\pi_{\theta,\beta_K}(f)-\log Z_{\theta,\beta_K}.
\]
Since relative entropy is nonnegative,
\[
\pi_{\theta,\beta_K}(f)\leq-\frac{1}{\beta_K}\log Z_{\theta,\beta_K}.
\]
Combining this inequality with \eqref{eq:proof_partition_lower} yields, for every $\epsilon>0$,
\begin{equation}
\pi_{\theta,\beta_K}(f)-f_{\theta,\rho}^\star\leq \epsilon+\frac{1}{\beta_K}\log\frac{1}{a_{\theta,\epsilon}}.
\label{eq:proof_temperature_epsilon}
\end{equation}

Proposition~\ref{prop:fk_feasibility} gives $\pi_{\theta,\beta_K}\ll p_\theta$. Since
\[
f(x)\geq f_{\theta,\rho}^\star \qquad p_\theta\text{-a.s.},
\]
the same inequality holds $\pi_{\theta,\beta_K}$-almost surely. Consequently,
\[
\left|\pi_{\theta,\beta_K}(f)-f_{\theta,\rho}^\star\right|=\pi_{\theta,\beta_K}(f)-f_{\theta,\rho}^\star.
\]
Taking the infimum of \eqref{eq:proof_temperature_epsilon} over $\epsilon>0$ gives
\begin{equation}
I_\beta\leq \inf_{\epsilon>0}\left\{\epsilon+\frac{1}{\beta_K}\log\frac{1}{a_{\theta,\epsilon}}\right\}.
\label{eq:proof_temperature_bound}
\end{equation}

The right-hand side vanishes as $\beta_K\rightarrow\infty$. Indeed, for any $\delta>0$, choose $\epsilon\in(0,\delta)$. Since $\epsilon$ is fixed and $a_{\theta,\epsilon}>0$, for all sufficiently large $\beta_K$,
\[
\frac{1}{\beta_K}\log\frac{1}{a_{\theta,\epsilon}}\leq\delta-\epsilon.
\]
Thus
\begin{equation}
\lim_{\beta_K\rightarrow\infty}\inf_{\epsilon>0}\left\{\epsilon+\frac{1}{\beta_K}\log\frac{1}{a_{\theta,\epsilon}}\right\}=0.
\label{eq:proof_temperature_vanishes}
\end{equation}

\paragraph{Step 2: Finite-particle Feynman--Kac error.}
By Assumption~\ref{ass:practical_error},
\begin{equation*}
    \left(
        \mathbb E
        \left[
            \left|
                \eta_K^M(f)
                -
                \pi_{\theta,\beta_K}(f)
            \right|^2
        \right]
    \right)^{1/2}
    \leq
    \frac{C_{K,f}}{\sqrt M}.
\end{equation*}
Consequently, by Cauchy--Schwarz,
\begin{align}
    I_M
    &=
    \left|
        \mathbb E[\eta_K^M(f)]
        -
        \pi_{\theta,\beta_K}(f)
    \right|\leq
    \mathbb E
    \left[
        \left|
            \eta_K^M(f)
            -
            \pi_{\theta,\beta_K}(f)
        \right|
    \right]
    \nonumber\\
    &\leq\left(
        \mathbb E
        \left[
            \left|
                \eta_K^M(f)
                -
                \pi_{\theta,\beta_K}(f)
            \right|^2
        \right]
    \right)^{1/2}
    \leq
    \frac{C_{K,f}}{\sqrt M}.
    \label{eq:proof_particle_bound}
\end{align}

\paragraph{Step 3: Coverage error and conclusion.}
By Assumption~\ref{ass:practical_error}, $p_\theta(\mathcal X_\rho\mid\rho)=1$. Since $f(x)\geq f_\rho^\star$, $x\in\mathcal X_\rho$, we have
\begin{equation*}
    f(x)\geq f_\rho^\star
    \qquad
    p_\theta\text{-a.s.}
\end{equation*}
and hence $f_{\theta,\rho}^\star \geq f_\rho^\star$. Therefore,
\begin{equation}
    I_{\mathrm{cov}}
    =
    \left|
        f_{\theta,\rho}^\star
        -
        f_\rho^\star
    \right|
    =
    f_{\theta,\rho}^\star
    -
    f_\rho^\star
    =
    \Delta_{\theta,\rho}.
    \label{eq:proof_coverage_bound}
\end{equation}

Substituting \eqref{eq:proof_temperature_bound}, \eqref{eq:proof_particle_bound}, and \eqref{eq:proof_coverage_bound} into the decomposition \eqref{eq:proof_error_decomposition} gives
\begin{align}
    \left|
        \mathbb E[\eta_K^M(f)]
        -
        f_\rho^\star
    \right|
    \leq\;&
    \Delta_{\theta,\rho}
    +
    \frac{C_{K,f}}{\sqrt M}+
    \inf_{\epsilon>0}
    \left\{
        \epsilon
        +
        \frac{1}{\beta_K}
        \log
        \frac{1}{a_{\theta,\epsilon}}
    \right\}.
\end{align}

Finally, for each fixed finite inverse-temperature schedule, $C_{K,f}<\infty$ is independent of $M$. Thus, letting $M\rightarrow\infty$ first removes the finite-particle term. Subsequently, letting $\beta_K\rightarrow\infty$ and using \eqref{eq:proof_temperature_vanishes} yields
\begin{equation}
    \limsup_{\beta_K\rightarrow\infty}
    \limsup_{M\rightarrow\infty}
    \left|
        \mathbb E[\eta_K^M(f)]
        -
        f_\rho^\star
    \right|
    \leq
    \Delta_{\theta,\rho},
\end{equation}
which establishes \eqref{eq:asymptotic_coverage_gap}.
\section{Implementation Details and Hyper-parameters}
\label{app:implementation}

\subsection{Constraint-Only Data Generation}
Training decisions are generated and filtered using only constraint information. Each dataset contains 10,000 retained samples. For every configuration, training risk levels are sampled uniformly from $[0,0.5]$; the risk levels reported in the experiments specify evaluation conditions.

Linear Gaussian feasibility is evaluated analytically. CCED uses two independent sets of $100{,}000$ wind scenarios for labeling and validation. Waveform data generation uses $1{,}000$ preliminary filtering scenarios, $4{,}000$ labeling scenarios, and $5{,}000$ validation scenarios without acceptance relaxation.

\subsection{Diffusion Training}

We use conditional 1D U-Nets for the linear Gaussian and robust waveform experiments and a conditional 2D U-Net for CCED-VPE-10. The diffusion models use $1{,}000$ forward diffusion steps with a linear noise-variance schedule from $10^{-4}$ to $0.02$. Training uses Adam with learning rate $10^{-4}$ and batch size $256$. The models are frozen before downstream optimization. ASQ and MW4 reuse the same frozen model as the original linear Gaussian experiment, without additional data generation or training. A single CCED model is likewise reused across all evaluated risk levels.
\subsection{Feynman--Kac Inference}

Post-specified objectives are introduced through derivative-free Feynman--Kac correction using only objective-function evaluations. We use $512$ particles for all inference experiments. The prescribed inverse-temperature schedule follows~\eqref{eq:prescribed_beta} and is synchronized with the reverse-diffusion steps, with $\beta_k$ specifying the objective tilt at step $k$. The schedule is fixed before sampling, while ESS is used only to determine when particle resampling is performed.

\subsection{Problem-Specific Settings}
\label{app:problem_setting}
\paragraph{Linear Gaussian and objective transfer.}
For the linear Gaussian benchmark, we use $\mathbf b=\bar{\mathbf c}=\mathbf 1$, $d=1$, $n=8$, and $\rho=0.1$, following~\citep{zhang2025ggdopt}. The linear Gaussian, ASQ, and MW4 experiments share the same chance constraint, whose feasibility is evaluated through its exact SOC reformulation. SOC-CVX provides the global reference.

\paragraph{Transfer experiment to ASQ.} The parameter in ASQ problem~\eqref{eq:asq_objective}, we set
$$\mathbf Q=\operatorname{diag}(0.6,0.8,1.0,1.2,1.4,1.6,1.8,2.0),$$
and
$$\boldsymbol\mu=
(-0.55,-0.45,-0.35,-0.25,-0.15,-0.05,0.05,0.15)^\top.
$$

\textbf{Transfer experiment to MW4.} In MW4 problem~\eqref{eq:mw4_objective}, the four quadratic branches use $\boldsymbol\mu_m=-0.25\mathbf 1+0.22\mathbf s_m$, $(\alpha_1,\alpha_2,\alpha_3,\alpha_4)=(0.025,0.015,0,0.035)$, and $\mathbf Q_m=\operatorname{diag}(\operatorname{roll}_{2(m-1)}(\mathbf q))$, where $q_j=0.8+0.4(j-1)/7$. The sign vectors $\mathbf s_m$ are specified by
$$
\mathbf S=
\begin{bmatrix}
1&1&1&1&-1&-1&-1&-1\\
1&1&-1&-1&1&1&-1&-1\\
1&-1&1&-1&1&-1&1&-1\\
1&1&-1&-1&-1&-1&1&1
\end{bmatrix}.
$$
For MW4, the global reference is obtained by solving the four convex quadratic branches separately under the shared chance constraint and taking the smallest optimum. Reference solutions are used only for post-hoc evaluation and are never supplied to the sampling methods.

\begin{table}[h]
\centering
\caption{Generator output limits and valve-point cost coefficients for CCED-VPE-10.}
\label{tab:cced_vpe_generator_parameters}
\begin{tabular}{cccccccc}
\toprule
Unit $i$ & $P_i^{\min} (MW)$ & $P_i^{\max} (MW)$ & $a_i$ & $b_i$ & $c_i$ & $d_i$ & $e_i$ \\
\midrule
1  & 10  & 55  & 1000.403 & 40.5407 & 0.12951 & 33 & 0.0174 \\
2  & 20  & 80  & 950.606  & 39.5804 & 0.10908 & 25 & 0.0178 \\
3  & 47  & 120 & 900.705  & 36.5104 & 0.12511 & 32 & 0.0162 \\
4  & 20  & 130 & 800.705  & 39.5104 & 0.12111 & 30 & 0.0168 \\
5  & 50  & 160 & 756.799  & 38.5390 & 0.15247 & 30 & 0.0148 \\
6  & 70  & 240 & 451.325  & 46.1592 & 0.10587 & 20 & 0.0163 \\
7  & 60  & 300 & 1243.531 & 38.3055 & 0.03546 & 20 & 0.0152 \\
8  & 70  & 340 & 1049.998 & 40.3965 & 0.02803 & 30 & 0.0128 \\
9  & 135 & 470 & 1658.569 & 36.3278 & 0.02111 & 60 & 0.0136 \\
10 & 150 & 470 & 1356.659 & 38.2704 & 0.01799 & 40 & 0.0141 \\
\bottomrule
\end{tabular}\end{table}
\paragraph{CCED-VPE-10.}
We use a ten-unit system with $P_D=2000$ MW, $W=100Z$ MW, $Z\sim\mathrm{Beta}(5,5)$, and $\rho=0.10$ in the main experiment. Table~\ref{tab:cced_vpe_generator_parameters} gives the generator limits and valve-point cost coefficients.

Under this uncertainty model, the chance constraint is equivalent to
\begin{equation} \sum_{i=1}^{N_G}p_i\geq P_D-100F_{\mathrm{Beta}(5,5)}^{-1}(\rho). \label{eq} \end{equation}

For CCED, physical outputs are normalized as $x_i=2(p_i-P_i^{\min})/(P_i^{\max}-P_i^{\min})-1$ and padded with zeros to form a $16$-dimensional input. The inference objective is scaled as $\widetilde C(\mathbf p)=(C(\mathbf p)-100000)/10000$. Reported costs always use the original, unscaled objective.

\subsection{Evaluation}
For each returned CCED decision $\mathbf p^{(r)}$, its satisfaction probability is evaluated analytically as
\begin{equation} P_r=1-F_{\mathrm{Beta}(5,5)}\left(\frac{P_D-\sum_{i=1}^{N_G}p_i^{(r)}}{100}\right). \end{equation}
The cumulative distribution function is set to zero below its support and one above it. Mean Probability in Table~\ref{tab:cced_results} is the average of $P_r$ across independent runs; for PWL-MILP, it is the probability of its single solution. Costs and probabilities are evaluated on the same returned decisions. Generator bounds and $P_r\geq1-\rho$ are checked for every run with a numerical tolerance of $10^{-10}$. Higher satisfaction probabilities indicate larger reliability margins and do not necessarily imply lower cost.
\paragraph{Reported statistics.}
Each stochastic method is evaluated over $100$ independent runs per setting. We compute the reported mean, sample standard deviation, and, where applicable, median from the decisions returned by the $100$ runs. ``Best'' is the lowest cost among these returned decisions. PWL-MILP is solved once per risk level and therefore has no reported standard deviation.
\section{Additional Experimental Results}
\label{app:additional_results}
\begin{table*}[ht!]
\centering
\caption{Risk-level adaptation on CCED-VPE-10 at $\rho\in\{0.05,0.15,0.20\}$. We report mean $\pm$ sample standard deviation and best cost over $100$ runs, except PWL-MILP, which uses one solve per risk level. Bold costs indicate the best available learning-based results.}
\label{tab:cced_multirisk_all_methods}
\resizebox{\textwidth}{!}{\begin{tabular}{clcccc}
\toprule
$\rho$ & Method & Cost Mean $\downarrow$ & Cost Std $\downarrow$ & Cost Best $\downarrow$ & Mean Probability \\
\midrule
\multirow{6}{*}{$0.05$} & DE + exact projection & $104689.448$ & $0.000$ & $104689.448$ & $0.9500$ \\
& PWL-MILP + exact projection & $104689.483$ & -- & $104689.483$ & $0.9500$ \\
& 95\% CI-SAA-DE & $104717.493$ & $0.000$ & $104717.493$ & $0.9537$ \\
\cmidrule(lr){2-6}
& GGDOpt~\citep{zhang2025ggdopt} + VPE adaptation & $108951.188$ & $1344.889$ & $106546.719$ & $0.9729$ \\
& DiffOPT~\citep{kong2025constrained} & $105986.707$ & $166.335$ & $105554.119$ & $0.9500$ \\
& \textbf{D$^3$Opt} & $\mathbf{105111.845}$ & $\mathbf{112.923}$ & $\mathbf{104831.074}$ & $0.9500$ \\
\cmidrule(lr){1-6}
\multirow{6}{*}{$0.15$} & DE + exact projection & $104190.322$ & $0.000$ & $104190.322$ & $0.8500$ \\
& PWL-MILP + exact projection & $104190.504$ & -- & $104190.504$ & $0.8500$ \\
& 95\% CI-SAA-DE & $104197.316$ & $0.000$ & $104197.316$ & $0.8519$ \\
\cmidrule(lr){2-6}
& GGDOpt~\citep{zhang2025ggdopt} + VPE adaptation & $108734.612$ & $1421.301$ & $106437.094$ & $0.9331$ \\
& DiffOPT~\citep{kong2025constrained} & $105455.728$ & $162.442$ & $105014.749$ & $0.8500$ \\
& \textbf{D$^3$Opt} & $\mathbf{104590.257}$ & $\mathbf{113.299}$ & $\mathbf{104284.908}$ & $0.8500$ \\
\cmidrule(lr){1-6}
\multirow{6}{*}{$0.20$} & DE + exact projection & $104018.596$ & $0.000$ & $104018.596$ & $0.8000$ \\
& PWL-MILP + exact projection & $104018.932$ & -- & $104018.932$ & $0.8000$ \\
& 95\% CI-SAA-DE & $104020.643$ & $0.000$ & $104020.643$ & $0.8007$ \\
\cmidrule(lr){2-6}
& GGDOpt~\citep{zhang2025ggdopt} + VPE adaptation & $108670.804$ & $1445.010$ & $106261.099$ & $0.9166$ \\
& DiffOPT~\citep{kong2025constrained} & $105271.029$ & $159.120$ & $104860.692$ & $0.8000$ \\
& \textbf{D$^3$Opt} & $\mathbf{104449.668}$ & $\mathbf{110.163}$ & $\mathbf{104166.527}$ & $0.8000$ \\
\bottomrule
\end{tabular}}
\end{table*}

\subsection{Multi-Risk Generalization on CCED-VPE-10}
\label{app:cced_multirisk}
We investigate whether the frozen constraint-only diffusion model used in Section~\ref{sec:cced_vpe} can support different risk requirements without retraining. We retain the same dispatch objective, generator parameters, and Beta wind distribution, and evaluate $\rho\in \{0.05,0.10,0.15,0.20\}$. These test levels lie within the continuous risk interval used for training. At inference time, we change only the risk condition supplied to the frozen model, while feasibility is evaluated under the corresponding chance constraint.
\paragraph{Baselines and Evaluation}
We compare against DE with exact-feasibility projection and DiffOPT~\citep{kong2025constrained}. SAA-DE is solved separately for each risk level, while DiffOPT and our method reuse the same frozen checkpoint across all four levels. 

Each stochastic method is evaluated over $100$ independent runs per risk level. We report the mean, sample standard deviation, and best cost under the original valve-point objective, together with the number of returned solutions satisfying the analytic chance constraint under the Beta wind distribution.

\paragraph{Results}
Table~\ref{tab:cced_multirisk_all_methods} shows that our method achieves lower mean and best costs than DiffOPT at every tested risk level, together with smaller cost standard deviations. All 100 returned solutions at each risk level satisfy the analytic chance constraint. As $\rho$ increases, the required thermal generation decreases, and the costs decrease accordingly.

These results demonstrate that the same frozen constraint-only model can support optimization under different risk requirements within the training interval without data regeneration or retraining. The traditional methods retain lower costs on this analytically tractable instance, while the comparison with DiffOPT supports the effectiveness of derivative-free Feynman--Kac correction under the shared checkpoint and evaluation protocol.

\begin{table}[h!]
\centering
\caption{Ablation experiments on CCED-VPE-10 at $\rho=0.1$, varying particle count $M\in\{64,256,512,1024,2048\}$ and terminal inverse temperature $\beta_{\max}\in\{0,5,10,20,50\}$. Results are reported over $100$ runs per setting.}
\label{tab:cced_ablation}
\begin{tabular}{cc|ccc}
\toprule
$\beta_{\max}$ & $M$ & FvalMean $\downarrow$ & FvalStd $\downarrow$ & FvalBest $\downarrow$ \\
\midrule
10 & \cellcolor{yellow!20}64 & $105161.669$ & $243.555$ & $104691.425$ \\
10 & \cellcolor{yellow!20}256 & $104887.677$ & $126.620$ & $104593.105$ \\
10 & \cellcolor{yellow!20}512 & $104823.577$ & $102.065$ & $104582.444$ \\
10 & \cellcolor{yellow!20}1024 & $104759.778$ & $90.430$ & $104549.575$ \\
10 & \cellcolor{yellow!20}2048 & $\mathbf{104707.598}$ & $\mathbf{74.224}$ & $\mathbf{104521.263}$ \\
\midrule
\cellcolor{blue!10}0 & 512 & $107183.587$ & $268.987$ & $106146.715$ \\
\cellcolor{blue!10}5 & 512 & $104977.214$ & $124.773$ & $104665.561$ \\
\cellcolor{blue!10}10 & 512 & $104823.577$ & $\mathbf{102.065}$ & $104582.444$ \\
\cellcolor{blue!10}20 & 512 & $104768.454$ & $104.858$ & ${104523.105}$ \\
\cellcolor{blue!10}50 & 512 & $\mathbf{104759.376}$ & $103.180$ & $\mathbf{104519.511}$ \\
\bottomrule
\end{tabular}
\end{table}
\subsection{Ablation Study}
\label{app:ablation}

We examine the effects of particle count $M$ and terminal inverse temperature $\beta_{\max}$ on CCED-VPE-10 at $\rho=0.1$. In each ablation, we vary one parameter while keeping the other and all remaining settings fixed. The setting $\beta_{\max}=0$ disables objective tilting during reverse diffusion and therefore serves as an unguided baseline. Each configuration is evaluated over $100$ independent runs.

Table~\ref{tab:cced_ablation} shows that increasing the particle count consistently improves both solution quality and stability. As $M$ increases from $64$ to $2048$ with $\beta_{\max}=10$, the mean cost decreases from $105161.669$ to $104707.598$, while the sample standard deviation decreases from $243.555$ to $74.224$. The best cost also improves from $104691.425$ to $104521.263$. These results indicate that a larger particle population improves coverage of low-objective regions and reduces finite-particle variability.

Increasing $\beta_{\max}$ also strengthens objective optimization. With $M=512$, raising $\beta_{\max}$ from $0$ to $50$ reduces the mean cost from $107183.587$ to $104759.376$ and the best cost from $106146.715$ to $104519.511$. The standard deviation decreases substantially up to $\beta_{\max}=10$ and changes only slightly at $\beta_{\max}=20, 50$. This trend shows that stronger objective tilting moves the particle system toward lower-cost regions, while overly strong concentration may trade stability for additional objective improvement.

The observed trends are consistent with Proposition~\ref{prop:optimization_error}: increasing $M$ reduces the finite-particle contribution, while increasing $\beta_{\max}$ reduces the finite-temperature contribution. We use $M=512$ and $\beta_{\max}=10$ as the default configuration because it provides a favorable balance between computational cost, solution quality, and stability. Larger settings further improve the mean or best cost but require more particles or produce slightly higher variability.

\begin{table*}[t!]
\centering
\caption{Robust waveform results with $N_t=16$ and $K=3$. Learning-based power statistics are from 100 runs. The final column reports per-user chance-constraint satisfaction probabilities on the returned decisions, averaged across runs for learning-based methods. For a traditional baseline it refers to its single solution. Bold power values are the lowest mean and best among learning-based methods.}
\label{tab:app-waveform-unified}
\resizebox{\textwidth}{!}{\begin{tabular}{clcccc}
\toprule
$\rho$ & Method & Power Mean $\downarrow$ & Power Std $\downarrow$ & Power Best $\downarrow$& $P_1/P_2/P_3$  \\
\midrule
\multirow{6}{*}{$0.05$} & Sphere Bounding~\citep{wang2014outage} & $0.1374$ & -- & -- & $(0.99;0.99;0.99)$\\
& Bernstein-type Inequality~\citep{wang2014outage} & $0.1260$ & -- & --& $(0.97;0.98;0.95)$ \\
\cmidrule(lr){2-6}
& GGDOpt First-order~\citep{zhang2025ggdopt} & $0.127992$ & $0.000205$ & $0.127339$& $(0.99;0.98;0.95)$ \\
& GGDOpt Second-order~\citep{zhang2025ggdopt} & $0.127989$ & $0.000206$ & $0.127335$& $(0.98;0.98;0.95)$  \\
& DiffOPT~\citep{kong2025constrained} & $0.246433$ & $0.000565$ & $0.246110$& $(0.99;0.98;0.97)$ \\
& \textbf{D$^3$Opt} & $\mathbf{0.127732}$ & $0.000209$ & $\mathbf{0.127275}$& $(0.97;0.95;0.96)$ \\
\cmidrule(lr){1-6}
\multirow{6}{*}{$0.10$} & Sphere Bounding~\citep{wang2014outage} & $0.1366$ & -- & -- & $(0.99;0.99;0.99)$\\
& Bernstein-type Inequality~\citep{wang2014outage} & $0.1253$ & -- & --&  $(0.93;0.94;0.93)$\\
\cmidrule(lr){2-6}
& GGDOpt First-order~\citep{zhang2025ggdopt} & $0.127941$ & $0.000251$ & $0.127019$& $(0.92;0.91;0.97)$ \\
& GGDOpt Second-order~\citep{zhang2025ggdopt} & $0.127938$ & $0.000252$ & $0.127009$& $(0.95;0.92;0.92)$ \\
& DiffOPT~\citep{kong2025constrained} & $0.237796$ & $0.000598$ & $0.236225$& $(0.97;0.94;0.94)$ \\
& \textbf{D$^3$Opt} & $\mathbf{0.126043}$ & $0.000187$ & $\mathbf{0.125518}$& $(0.92;0.90;0.91)$ \\
\cmidrule(lr){1-6}
\multirow{6}{*}{$0.15$} & Sphere Bounding~\citep{wang2014outage} & $0.1361$ & -- & -- & $(0.99;0.99;0.99)$\\
& Bernstein-type Inequality~\citep{wang2014outage} & $0.1248$ & -- & --& $(0.90;0.90;0.92)$ \\
\cmidrule(lr){2-6}
& GGDOpt First-order~\citep{zhang2025ggdopt} & $0.127983$ & $0.000187$ & $0.127474$ & $(0.90;0.87;0.92)$\\
& GGDOpt Second-order~\citep{zhang2025ggdopt} & $0.127981$ & $0.000189$ & $0.127467$& $(0.90;0.87;0.86)$ \\
& DiffOPT~\citep{kong2025constrained} & $0.245401$ & $0.000475$ & $0.245118$& $(0.90;0.92;0.89)$ \\
& \textbf{D$^3$Opt} & $\mathbf{0.125047}$ & $0.000177$ & $\mathbf{0.124726}$& $(0.85;0.87;0.86)$ \\
\cmidrule(lr){1-6}
\multirow{6}{*}{$0.20$} & Sphere Bounding~\citep{wang2014outage} & $0.1357$ & -- & --& $(0.99;0.99;0.99)$ \\
& Bernstein-type Inequality~\citep{wang2014outage} & $0.1244$ & -- & --& $(0.90;0.89;0.91)$  \\
\cmidrule(lr){2-6}
& GGDOpt First-order~\citep{zhang2025ggdopt} & $0.127964$ & $0.000245$ & $0.127006$ & $(0.90;0.82;0.88)$ \\
& GGDOpt Second-order~\citep{zhang2025ggdopt} & $0.127962$ & $0.000247$ & $0.126994$ & $(0.81;0.85;0.88)$\\
& DiffOPT~\citep{kong2025constrained} & $0.211671$ & $0.000463$ & $0.210839$& $(0.88;0.87;0.91)$ \\
& \textbf{D$^3$Opt} & $\mathbf{0.124361}$ & $0.000184$ & $\mathbf{0.123822}$& $(0.85;0.82;0.80)$ \\
\bottomrule
\end{tabular}}
\end{table*}
\subsection{Robust Waveform Experiments}
\label{app:robust_waveform}
We further evaluate the method on robust multi-user waveform design with $N_t=16$ transmit antennas and $K=3$ users. Let $\mathbf{w}_k\in\mathbb{R}^{N_t}$ denote the beamforming vector for user $k$, and let $\mathbf{W}=(\mathbf{w}_1,\ldots,\mathbf{w}_K)$. The problem minimizes total transmission power subject to individual signal-to-interference-plus-noise ratio chance constraints:
\begin{align}
\min_{\mathbf{W}}\quad & f(\mathbf{W})=\sum_{k=1}^{K}\|\mathbf{w}_k\|_2^2, \label{eq:app_waveform}\\
\text{s.t.}\quad & \mathbb{P}\left\{\frac{(\mathbf{h}_k^\top\mathbf{w}_k)^2}{\sum_{j\ne k}(\mathbf{h}_k^\top\mathbf{w}_j)^2+\sigma_i^2}\geq\gamma\right\}\geq1-\rho,\qquad k=1,\ldots,K. \nonumber
\end{align}
The uncertain channel is modeled as $\mathbf{h}_k=\bar{\mathbf{h}}_k+\mathbf{e}_k$, where $\mathbf{e}_k\sim\mathcal{N}(\mathbf{0},\sigma_e^2\mathbf{I})$. Following GGDOpt~\citep{zhang2025ggdopt}, we set $\gamma=0.5$, $\sigma_e^2=0.002$, and $\sigma_i=1$, and evaluate $\rho\in\{0.05,0.10,0.15,0.20\}$.
\paragraph{Baselines and Evaluation}
We compare against Sphere Bounding~\citep{wang2014outage}, Bernstein-type Inequality~\citep{wang2014outage}, first- and second-order GGDOpt~\citep{zhang2025ggdopt}, and DiffOPT~\citep{kong2025constrained}. DiffOPT uses iterative importance sampling, whereas D$^3$Opt applies objective-value-based Feynman--Kac correction during reverse diffusion.

Each learning-based method is evaluated over $100$ runs. We report mean, sample standard deviation, and best power. For user $k$, let $\hat P_{r,k}$ be the fraction of test scenarios satisfying its SINR constraint for returned decision $r$. Table~\ref{tab:app-waveform-unified} reports $\bar P_k=100^{-1}\sum_{r=1}^{100}\hat P_{r,k}$ for learning-based methods; traditional-baseline triplets refer to one solution. These run averages do not give the number of individually feasible returns.

\textbf{Results} Table~\ref{tab:app-waveform-unified} shows that D$^3$Opt achieves the lowest mean and best transmission power among the evaluated learning-based methods at all four risk levels. The minimum of its three run-averaged user satisfaction probabilities is $0.95$, $0.90$, $0.85$, and $0.80$ for $\rho=0.05$, $0.10$, $0.15$, and $0.20$, respectively, matching the nominal $1-\rho$ thresholds at the reported two-decimal precision. Relative to second-order GGDOpt, D$^3$Opt reduces mean power by $0.20\%$, $1.48\%$, $2.29\%$, and $2.81\%$. For D$^3$Opt, mean power decreases monotonically as $\rho$ increases. These results are consistent with using the allowed risk budget closely across the tested risk levels without retraining.

\end{document}